\documentclass[opre,nonblindrev]{informs3_no_remark} 

\OneAndAHalfSpacedXI 

\usepackage{endnotes}
\let\footnote=\endnote

\usepackage{natbib}
 \bibpunct[, ]{(}{)}{,}{a}{}{,}%
 \def\bibfont{\small}%
\TheoremsNumberedThrough     
\ECRepeatTheorems

\EquationsNumberedThrough    

\usepackage{booktabs}       
\usepackage{amsfonts}       
\usepackage{nicefrac}       
\usepackage[caption=false]{subfig}
\usepackage{amsmath,amsfonts,amssymb, graphicx, setspace, verbatim}
\usepackage{multirow}

\usepackage{enumitem}
\usepackage{float}
\usepackage[]{algorithm2e}
\usepackage{dsfont}

\def\PP{{\mathbb P}}
\def\E{{\mathbb E}}

\def\S{{\mathcal S}}

\newcommand{\Var}{\text{Var}}
\newcommand{\Cov}{\text{Cov}}

\begin{document}
\RUNAUTHOR{Yi Zhu, Jing Dong and Henry Lam}
	
\RUNTITLE{Inference and Exploration for Reinforcement Learning}

\TITLE{Uncertainty Quantification and Exploration for Reinforcement Learning}

%

\ARTICLEAUTHORS{%
	\AUTHOR{Yi Zhu}
	\AFF{Department of Industrial Engineering and Management Sciences,
	Northwestern University, Evanston, IL 60208, \EMAIL{yizhu2020@u.northwestern.edu}} 
	\AUTHOR{Jing Dong}
	\AFF{Division, Risk and Operations Division, Columbia Business School, New York, NY 10027, \EMAIL{jing.dong@gsb.columbia.edu}}
	\AUTHOR{Henry Lam}
	\AFF{Department of Industrial Engineering and Operations Research, Columbia University, New York, NY 10027, \EMAIL{henry.lam@columbia.edu}}
} 


\ABSTRACT{%
	We investigate statistical uncertainty quantification for reinforcement learning (RL) and its implications in exploration policy. Despite ever-growing literature on RL applications, fundamental questions about inference and error quantification, such as large-sample behaviors, appear to remain quite open. In this paper, we fill in the literature gap by studying the central limit theorem behaviors of estimated Q-values and value functions under various RL settings. In particular, we explicitly identify closed-form expressions of the asymptotic variances, which allow us to efficiently construct asymptotically valid confidence regions for key RL quantities. Furthermore, we utilize these asymptotic expressions to design an effective exploration strategy, which we call Q-value-based Optimal Computing Budget Allocation (Q-OCBA). The policy relies on maximizing the relative discrepancies among the Q-value estimates. Numerical experiments show superior performances of our exploration strategy than other benchmark policies.
}%



\maketitle
\section{Introduction}

We consider the standard reinforcement learning (RL) setting where the agent 
interacts with a random environment and aims to maximize the accumulated discounted reward over time. The environment is formulated as a Markov decision process (MDP) and the agent is uncertain
about the true dynamics to start with. As the agent interacts with the environment, 
data about system dynamics are collected and the agent becomes increasingly confident about her decision. With finite data, however, the potential reward from each decision is estimated with errors. The goal of the agent is to make decisions as close to optimal as possible over time, given the collected data at hand.

In this paper, we focus on statistical uncertainty quantification in RL and its implications in exploration policy. More precisely, we investigate the large-sample behaviors in estimating the so-called Q-values (action-value functions) and the associated optimal value functions, in the form of central limit convergences to Gaussian (or other) distributions. We explicitly characterize the asymptotic variances of the limiting distributions in terms of the system parameters (i.e., transition probabilities, mean rewards, and data-collection policies), from which we can construct statistically valid confidence intervals/regions. These closed-form formulas, moreover, are utilized in a novel procedure that we call Q-value-based Optimal Computing Budget Allocation (Q-OCBA) to efficiently explore the underlying MDP. 

The motivation for our investigation is twofold. First, like classical statistical inference, large-sample behavior and confidence interval construction are fundamental in assessing the error and reliability of estimated quantities with respect to the data noises. Despite the ever-growing literature on RL applications, some of these fundamental questions appear to remain quite open. Our first goal in this paper is to fill in the literature gap by studying the large sample behaviors of RL. In the RL context, the key objects that measure performances and determine optimal policies are the Q-values, namely the maximum cumulative rewards that initialize at given state-and-action pairs. Closely related and derivable from the Q-values are the optimal value functions, namely the maximum cumulative rewards initialized at given states. We derive appropriate large-sample asymptotic distributions that allow us to construct confidence regions for Q-value and  optimal-value-function estimations when data are collected from the underlying Markov chain. In addition to quantifying the error of each value estimate, our results also allow one to evaluate the assertiveness of performances among different decisions. For example, while one state-action pair may appear better by reading the point estimate of its Q-value, its variability, which is estimated via our approach, can also be larger, and should be properly accounted for when deciding whether the action should be selected given the state.



Our investigation on the limit theorems and confidence region constructions is in line with the so-called input uncertainty problem in stochastic simulation \citep{barton2012tutorial,song2014advanced,lam2016advanced}.
The latter problem aims to quantify the impacts on simulation outputs due to the statistical noises incurred in fitting the input models that generate the fed-in random variates. Quantifying these impacts requires estimating the output variability contributed from the input noises (e.g., \citealt{song2015quickly,cheng1997sensitivity}), or constructing output confidence intervals that properly account for input uncertainty (e.g., \citealt{cheng2004calculation,barton2013quantifying}). In stochastic simulation, the inputs often take the form of parametric models (e.g., \citealt{xie2014bayesian,zhu2020risk}), a finite collection or mixture of parametric models (e.g., \citealt{chick2001input,zouaoui2004accounting}), or in the nonparametric case, empirical distributions where independence of random variates is typically assumed (e.g., \citealt{barton2001resampling,yi2017efficient}). Our study can be viewed as addressing the input uncertainty problem for MDPs, where the input model now takes the form of Markov transition matrices and reward distributions. In addition to the Markovian structure, another distinction of our work from conventional input uncertainty is the involvement of optimization in our target quantity (both the Q-values and the optimal value functions). In this regard, our work also relates to the uncertainty quantification of stochastic optimization problems, which aims to understand the sub-optimality of obtained solutions due to the statistical noise in constructing the objective functions \citep{mak1999monte,bayraksan2006assessing}. Similar to the sample average approximation and stochastic programming literature (e.g., \citealt{shapiro2014lectures,higle2013stochastic}), our limit theorems capture the statistical fluctuations of estimated optimal values. However, instead of a finite number of stages or having independent variates across stages, we focus on the error quantification of maximum cumulative rewards, which necessitate the intricate use of the Bellman equation machinery and the stationarity of the underlying Markov chain. 

Besides input uncertainty, we describe our methodological advancements relative to the RL literature. Our results complement the existing finite-sample bounds (e.g., \citealt{Kearns1998,Kakade:2003,munos2008finite}) by offering closed-form asymptotic variances that often show up in the first-order terms in these bounds. In constructing confidence intervals (and also in using variance-based exploration strategies like Q-OCBA that we detail further below), finite-sample error bounds are typically conservative. In contrast, in our development, the dependence of asymptotic variance on key parameters, such as the reward and transition variabilities, and the frequency of visits for each state-action pair, is tightly characterized. To this end,  \citet{mannor2004bias,mannor2007bias} are among the first and very few to study asymptotically tight statistical properties of value function estimation. They investigate the bias and variance in value function estimates under a fixed policy.  Methodologically, our technique resolves a main technical challenge in \citet{mannor2004bias,mannor2007bias} that allows us to generalize their variance results to Q-values and optimal value functions. The derivation in  \citet{mannor2004bias,mannor2007bias} hinges on an expansion of the value function in terms of the perturbation of the transition matrix, which (as pointed out by the authors) is not easily extendable from the value function under a fixed-policy to the optimal value function. In contrast, our results utilize an implicit function theorem applied to the Bellman equation which can be verified to be sufficiently smooth. This idea turns out to allow us to obtain gradients for Q-values, translate the asymptotic variance of Q-values to the optimal value function, and furthermore generalize similar asymptotic distributional results for constrained MDPs, approximate value iteration, and kernel estimation. Recently, \cite{devraj2017zap} and \cite{chen2020explicit} derive central limit convergence results and characterize the corresponding asymptotic variances for some parametric Q-learning algorithms. Their development builds on an elegant connection between Q-learning and stochastic approximation. On the other hand, our main asymptotic result (Theorem \ref{thm:clt_basic}) does not impose any parametric assumption on the Q-values.

Our work is also related to the line of studies on dynamic treatment regimes (DTR) applied commonly in medical decision-making \citep{laber2014dynamic}. DTR focuses on the statistical properties of policies on finite horizons (such as two-period). 
Our infinite-horizon results on the optimal values and Q-values distinguish our developments from the DTR literature. 
Indeed, \citet{laber2014dynamic} list the statistical properties for the infinite-horizon case as an ``open problem''. 
Moreover, as we will elaborate in the sequel, our result for the non-unique policy case can be demonstrated to correspond to the ``non-regularity'' concept in DTR, where the true parameters are very close to the decision ``boundaries" that switch the optimal policy (motivated by situations of small treatment effects), thus making the obtained policy highly sensitive to estimation noises. 



The second motivation of our study is to design good exploration policies by directly using our tight error estimates derived from our limit theorems. We consider the pure exploration setting where an agent is first assigned a period to collect as much experience as possible, and then, with the optimal policy trained offline, starts deployment to gain reward. The goal is to collect as informative data as possible in the first stage via a good exploration strategy such that, at the end of this stage, our recommended policy coincides with the true optimal policy of the MDP with the highest probability.
This setting is motivated by recent autonomous-driving applications \citep{kalashnikov2018qt}, and differs from classical learning settings
that balance the exploration-exploitation trade-off: exploring poorly-understood states and actions in the hope of improving future performance versus exploiting existing knowledge to attain better reward now \citep{jaksch2010near}. Crucially, our setting has a dedicated learning stage in which no rewards are collected, and the focus is to select the best policy using the data collected in the learning stage.  The separation of learning and execution stages is particularly attractive when extra safety-protection tools could be incorporated into the learning stage.
For example, to derive safe driving policies for autonomous vehicles,  ``soft crash" (or ``near crash") perceived by high-precision sensors is used to reveal the potential risks of unsafe driving behaviors and avoids real physical damage to the car \citep{kiran2020deep, amini2020learning}. 
The learning stage can also sometimes be carried out through a simulation platform where real maneuvers could be virtually simulated \citep{schoner2017role,corso2020survey}.



Our main contribution in regard to the second motivation is the design of a pure exploration policy by maximizing the worst-case relative discrepancy among the estimated Q-values (ratio of the mean squared difference to the variance). We coin our strategy Q-OCBA, as our policy search criterion resembles the widely known Optimal Computing Budget Allocation (OCBA) procedure in simulation optimization \citep{chen2011stochastic}.
OCBA has been successfully applied to ranking and selection \citep{glynn2004large,gao2017robust} where the search space is a finite number of discrete simulatable alternatives, a variant called targeting and selection \citep{ryzhov2018local}, random search \citep{chen2013optimal}, stochastically constrained problems and feasibility determination  \citep{lee2012approximate,gao2016efficient}, and quantile-based selection \citep{peng2018efficient}. In sequential OCBA implementations, one divides the computation budget into stages in which one sequentially updates mean and variance estimates, and optimizes next-stage budget allocations according to the worst-case relative discrepancy criterion \citep{chen2006efficient}. Like OCBA, the worst-case relative discrepancy in Q-OCBA serves as a proxy for the probability of selecting the optimal policy, which is analogous to the probability of correct selection in ranking and selection. Nonetheless, Q-OCBA bears two important distinctions: One is that we utilize the Q-value estimates and asymptotic variance estimates, derived via our limit theorems, in our maximization criterion. The other is that while the allocation in OCBA can be derived and implemented straightforwardly, it is not the case in RL. Because of the underlying Markov chain, the desired allocation has to be implemented through a randomized policy that gives rise to a stationary distribution of state-action pairs that matches the desired allocation. To address this issue, we design a computationally tractable scheme to obtain the randomized policy. More precisely, through our derived asymptotic variance formula that depends explicitly on the frequency of visits to each state-action pair, and by characterizing all the feasible stationary distributions of the MDP as a set of linear constraints, we can derive our Q-OCBA exploration policy by solving a convex optimization problem. We further propose a sequential updating rule to implement the exploration strategy in practice. 

To the best of our knowledge, our work is among one of the few to concretely use the variances of Q-value estimates and link budget allocation to develop exploration policies for MDP. The only other work that uses OCBA-based criterion for efficient budget allocation in MDP is \citet{jia2012efficient}, which is for a different setting. Similar to our work, the objective is to maximize the probability of selecting the optimal action at each state. However, they consider a setting where the state-action pair at each ``step" can be arbitrarily chosen and the Q-values for different state-action pairs are estimated independently by simulation. This allows them to use standard sample variance for variance estimation. In contrast, we consider the ``online" setting where the next state is determined by the current state-action pair through the underlying Markov chain. This poses a substantial challenge to constructing a good variance estimator. 
Interestingly, one way to further improve the efficiency of the strategy proposed in \citet{jia2012efficient} is to reuse (at least partially) the sample simulated for the Q-value estimation at one state-action pair to estimate the Q-value at another state-action pair. However, this would require sophisticated derivations of the corresponding variances as in our setting. Another related work is \cite{devraj2017zap}, which uses asymptotic variances to optimize the corresponding Q-learning algorithm with certain parameterization. They consider the classic regret minimization objective, which is different from our pure exploration setting.

More generally, we also position Q-OCBA relative to the online learning literature. In particular, our pure exploration problem resembles the so-called best-arm identification in multi-armed bandits \citep{audibert2010best,kaufmann2016complexity, russo2016simple},
which studies procedures to efficiently identify alternatives with the highest reward. This problem is intricately related to ranking and selection in simulation \citep{kim2007recent}
in that they both focus on the performance of the selected alternative at the end of the learning period, without concerning the interim rewards, but with differences in the analytical tools and the types of complexity guarantees. The ranking and selection or best-arm identification problems can be viewed as a one-state special case of the pure exploration setting in RL, which is much less studied (the only recent work that we are aware of is \citet{putta2017pure}). Compared to the former problem that often assumes independence across samples from different alternatives, RL is more complex due to the long-term effect of an action that could determine the future frequency of visits for many state-action pairs. Q-OCBA provides an implementable avenue to obtain a good policy for precisely this purpose. Recently, there are also extensions of ranking and selection to context-dependent settings, where rewards can depend on agent's covariates \citep{shen2021ranking}. Although in both RL and contextual ranking and selection, the rewards are state-dependent, in RL, the state variables follow a Markov chain, whereas in contextual ranking and selection, the state variables (covariates) are static. 

Most existing RL algorithms consider settings where one solves for the optimal policy while simultaneously learning the dynamics of the underlying MDP. The objective is to minimize the regret, which gives rise to a non-trivial exploration-exploitation trade-off. In this context, policies that focus on more systematic exploration such as posterior sampling reinforcement learning (PSRL) and various versions of the upper confidence bound algorithm (e.g., UCRL2, UCBVI) have been developed and shown to achieve good regret bounds \citep{osband2013more, azar2017minimax, jin2018q}. Recently, \cite{bellemare2017distributional}  advocates learning an approximate value distribution rather than its expectation to obtain more stable learning. In this paper, we consider a very different setting -- a pure exploration task.
Thus, our policy focuses more on exploration than the aforementioned policies (which were designed for a different objective). To support our approach, we show through numerical experiments that Q-OCBA consistently achieves superior performances compared to an array of benchmark policies in the literature (see Section \ref{sec:num_compare}).

The rest of the paper is organized as follows. Section \ref{sec:problem} first describes our MDP setup. Section \ref{sec:ucq} presents our results on large-sample behaviors. Section \ref{sec:explore} demonstrates their use in developing exploration strategies. Section \ref{sec:approx} studies two approximations schemes to deal with large state and/or action spaces. Section \ref{sec:numerics} substantiates our findings with simulation experiments. Section \ref{sec:con} provides some concluding remarks. All proofs are deferred to the appendix. 


\subsection{Notations}

We introduce some notations used throughout the paper. We denote ``$\Rightarrow$'' as ``convergence in distribution'', and $\mathcal{N}(\mu,\Sigma)$ as a multivariate Gaussian distribution with mean vector $\mu$ and covariance matrix $\Sigma$. We write $I$ as the identity matrix, and $e_i$ as the $i$-th unit vector. We also write ${\bf 0}$ as a zero vector or matrix. The dimension of $\mathcal{N}(\mu,\Sigma)$, $I$, $e_i$, and ${\bf 0}$ should be clear from the context. When not specified, all the vectors are column vectors. For $\epsilon>0$, we define $a\pm\epsilon:=[a-\epsilon,a+\epsilon]$, i.e., it is a closed interval with lower bound $a-\epsilon$ and upper bound $a+\epsilon$. Lastly, $\mathds{1}(\cdot)$ denotes the indicator function.

\section{Problem Setup}\label{sec:problem}
Consider an infinite horizon accumulated discounted reward MDP, 
$\mathcal{M} = (\mathcal{S}, \mathcal{A}, R, P, \gamma, \rho)$,
where $\mathcal{S}$ is the state space, $\mathcal{A}$ is
the action space, $R(s,a)$ denotes the random reward when the agent is in state $s\in\mathcal{S}$ and selects action $a\in\mathcal{A}$,
$P(s'|s,a)$ is the probability of transitioning to state $s^{\prime}$ in the next epoch given the current 
state $s$ and taken action $a$, $\gamma$ is the discount factor, and $\rho$ is the initial
state distribution. The distribution of the reward  $R$ and the transition probability $P$ are unknown to the agent.

We assume both $\mathcal{S}$  and $\mathcal{A}$ are finite sets.  
Without loss of generality, we denote $\mathcal{S} = \{1,2, \dots, m_s \}  \mbox{ and } \mathcal{A} = \{1,2, \dots, m_a \}$. 
We also make the following stochasticity assumption:
\begin{assumption}\label{assum:ind}
	$R(s,a)$ has a finite mean $\mu_R(s,a)$ and a finite variance $\sigma_R^2(s,a)$,  $\forall$ $(s,a)\in \mathcal{S}\times \mathcal{A}$. For any given $(s,a)\in\mathcal S\times \mathcal A$, 
	$R(s,a)$ and $s^{\prime} \sim P(\cdot|s,a)$  are independent random variables.
\end{assumption}

A policy $\pi$ is a mapping from each state $s\in \mathcal{S}$ to a probability measure  over actions in $\mathcal{A}$.
Specifically, we write $\pi(a|s)$ as the probability of taking action $a$ when the agent is in state $s$ and $\pi(\cdot|s)$ as the $m_a$-dimensional vector of action probabilities at state $s$. For convenience, we sometimes write $\pi(s)$ as the realized action given the current state $s$. 
The value function associated with a policy $\pi$ is defined as 
\[V^{\pi}(s) =\mathbb{E}^{\pi}\left[ \sum^{\infty}_{t=0} \gamma^t R(s_t, \pi(s_t)) |s_0 =s\right]\] 
with $s_{t+1} | (s_t, \pi(s_t)) \sim P(\cdot|s_t, \pi(s_t))$. The expected value function, under the initial distribution $\rho$, is denoted by 
\[\chi^{\pi}= \sum_{s} \rho(s) V^{\pi}(s).\]
A policy $\pi^{*}$ is said to be optimal if $V^{\pi^{*}}(s) = \max_{\pi} V^{\pi}(s)$ for all $s\in \mathcal{S}$. For convenience, we denote $V^*=V^{\pi^*}$ and $\chi^*=\sum_{s} \rho(s) V^*(s)$. The Q-value at $(s,a) \in \mathcal{S} \times \mathcal{A}$, denoted by $Q(s,a)$, is defined as  $Q(s,a)= \mu_R(s,a) + \gamma \mathbb{E}[V^{*}(s')|s,a]$. Correspondingly, $V^*(s) = \max_{a} Q(s,a)$ and the Bellman equation for the Q-values takes the form
\begin{eqnarray}\label{eq:bellman}
Q(s,a) = \mu_R(s,a) + \gamma \mathbb{E}\left[\max_{a'} Q(s',a')|s,a\right].
\end{eqnarray}
Denoting the Bellman operator as $\mathcal{T}_{\mu_R,P}(\cdot)$, $Q$ is a fixed point associated with  $\mathcal{T}_{\mu_R,P}$, i.e., $Q=\mathcal{T}_{\mu_R,P}(Q)$.
For most part of this paper, we make the following assumption about $Q$:
\begin{assumption} \label{assum:unique}
	For any state $s\in \mathcal{S}$, $\arg\max_{a\in \mathcal{A}} Q(s,a)$ is unique.
\end{assumption}
Under Assumption \ref{assum:unique}, the optimal policy $\pi^*$ is unique and deterministic. Let 
\[a^*(s)=\arg\max_{a\in \mathcal{A}} Q(s,a).\] 
Then $\pi^*(a|s)=\mathds{1}\left(a=a^*(s)\right)$. We relax this assumption in Section \ref{sec:non-unique}.

We next introduce some statistical quantities arising from data. 
Suppose we have $n$ observations (whose collection mechanism will be made precise later), which we denote as 
\[\{(s_t, a_t, r_t(s_t,a_t), s_t^{\prime}(s_t, a_t)): 1\leq t \leq n\},\] 
where $r_t(s_t,a_t)$ is the realized reward at time $t$ and $s_t^{\prime}(s_t, a_t)=s_{t+1}$. We define the sample mean $\hat{\mu}_{R,n}$ and the sample variance $\hat \sigma^2_{R,n}$ of the random rewards as
\begin{equation}
\hat{\mu}_{R,n}(s=i, a=j)= \begin{cases} \frac{\sum_{1\leq t \leq n} r_t(s_t, a_t)\mathds{1}(s_t=i, a_t=j)}{\sum_{1\leq t\leq n} \mathds{1}(s_t=i, a_t=j)}, \quad \text {if $\sum_{1\leq t\leq n} \mathds{1}(s_t=i, a_t=j)>0$ }\\ 
u_R^0(i,j) \quad \text {if $\sum_{1\leq t\leq n} \mathds{1}(s_t=i, a_t=j)=0$; }\\
\end{cases} \label{mu estimate}
\end{equation}
\begin{equation}
\hat{\sigma}^2_{R,n}(s=i,a=j)= \begin{cases}
\frac{\sum_{1\leq t \leq n} r_t(s_t, a_t)^2\mathds{1}(s_t=i, a_t=j)}{\sum_{1\leq t\leq n} \mathds{1}(s_t=i, a_t=j)}-\hat{\mu}_{R,n}(i, j)^2  \quad \text {if $\sum_{1\leq t\leq n} \mathds{1}(s_t=i, a_t=j)>0$ } \\
s_R^0(i,j) \quad \text {if $\sum_{1\leq t\leq n} \mathds{1}(s_t=i, a_t=j)=0$, }
\end{cases} \nonumber 
\end{equation}
where $u_R^0(i,j)\in \mathcal{R}$ and $s_R^0(i,j)\geq 0$ are some appropriate initial values for the mean and variance if no data is available.

Similarly, we define the empirical transition matrix $\hat{P}_n$ as
\begin{equation}
\hat{P}_n(s' =k | s=i, a=j) = \begin{cases}
\frac{\sum_{1\leq t\leq n} \mathds{1}(s_t=i, a_t=j, s_t'(s_t,a_t) = k)}{\sum_{1\leq t\leq n} \mathds{1}(s_t=i, a_t=j)} \quad \text {if $\sum_{1\leq t\leq n} \mathds{1}(s_t=i, a_t=j)>0$}\\
p^0(k|i,j) \quad \text {if $\sum_{1\leq t\leq n} \mathds{1}(s_t=i, a_t=j)=0$,}
\end{cases}\label{P estimate}
\end{equation}
where  $p^0$ with $p^0(k|i,j)\geq 0$ and  $\sum_k p^0(k|i,j)=1$ is some appropriate initialization of the transition matrix if no data is available. 

Note that $u^0(i,j)$, $s^0(i,j)$, and $p^0(k|i,j)$ above can be any reasonable initialization 
when the corresponding state-action pairs have not been visited yet. For example, we can initialize $\hat{P}_n(s' =k | s=i, a=j) = 1/m_s$, $\hat{\mu}_{R,n}(s=i, a=j)=0$, and $\hat{\sigma}^2_{R,n}(s=i,a=j)=1$ if there is no good prior knowledge.

Furthermore, we define the $m_s\times m_s$ sampling covariance matrix $\Sigma_{P_{s,a}}$ (with one sample point of $\mathds{1}(s_t=s, a_t=a)$) as
\begin{equation}\label{eq:Sigma_P}
\Sigma_{P_{s,a}}(k_1,k_2) = \begin{cases} P(k_1|s,a)(1-P(k_1|s,a))\quad k_1=k_2\\ 
-P(k_1|s,a) P(k_2|s,a) \quad k_1\neq k_2. 
\end{cases},\  \text{for\ }1\leq k_1\leq m_s,1\leq k_2\leq m_s. 
\end{equation}
With the data, we construct our estimate of $Q$, called $\hat{Q}_n$, via the empirical fixed point of the Bellman operator,
i.e., 
\[\hat{Q}_n =\mathcal{T}_{\hat{\mu}_{R,n}, \hat{P}_n}(\hat{Q}_n).\] 
Correspondingly, we write 
\[\hat{V}^{*}_n(s) = \max_{a\in\mathcal{A}} \hat{Q}_n(s,a) ~\mbox{ and }~
\hat{\chi}^{*}_n=\sum_{s\in\mathcal{S}}\rho(s)\hat{V}^{*}_n(s).\] 

Throughout the paper we focus on the estimation errors due to the noise of collected data, and assume the MDP or Q-value evaluation given the estimated average reward and transaction matrix can be obtained exactly.





\section{Quantifying Asymptotic Estimation Errors} \label{sec:ucq}


In this section, we present an array of results regarding the asymptotic behaviors of $\hat Q_n$ and $\hat V_n^*$ as $n\rightarrow\infty$. 
We start with the standard case where each state has a unique optimal action, i.e., Assumption \ref{assum:unique} holds.
We then discuss two important extensions: non-unique optimal action and constrained MDP. 

To prepare, we first make an assumption on our exploration policy $\pi$ to gather data. Define the extended transition probability $\tilde P^{\pi}$ under policy $\pi$ as 
\[\tilde P^{\pi}(s^{\prime}, a^{\prime}|s,a)=P(s^{\prime}|s,a)\pi(a^{\prime}|s^{\prime}).\] 

\begin{assumption} \label{assum:postive_stat}
	The Markov chain with transition probability $\tilde P^{\pi}$ is irreducible. 
\end{assumption}
Under Assumption \ref{assum:postive_stat}, as we have a finite state space and action space, $\tilde P^{\pi}$ has a unique stationary distribution, denoted by $w$, which is equal to the long run frequency of visiting each state-action pair, i.e., for $i\in \mathcal{S}$ and $j\in\mathcal{A}$ 
\[w(i,j)=\lim_{n\to \infty} \frac{1}{n}\sum_{1\leq t \leq n} \mathds{1}(s_t=i, a_t=j),\] 
where all $w(i,j)$'s are positive. Note that Assumption \ref{assum:postive_stat} is satisfied if for any two states $s,s'$, there exists a sequence of actions such that $s'$ is attainable from $s$ under $P$, and if $\pi$ is sufficiently mixed, e.g., $\pi$ satisfies $\pi(a|s)>0$ for all $s,a$. We define
\[\hat w_n(i,j) :=\frac{1}{n}\sum_{1\leq t \leq n} \mathds{1}(s_t=i, a_t=j),\]
which is the empirical frequency of visiting each state-action pair.

Let $N=m_sm_a$. In our subsequent algebraic derivations, we need to re-arrange $\mu_R$, $Q$, and $w$ as $N$-dimensional vectors. With a little abuse of notation, we define the following indexing rule:
$(s=i,a=j)$ is re-indexed as $(i-1)m_a+j$, e.g., 
\[\mu_R(i,j)=\mu_R((i-1)m_a+j).\]
We also need to re-arrange $\tilde P^{\pi}$ as an $N\times N$ matrix following the same indexing rule, e.g.,
\[\tilde P^{\pi}(i^{\prime}, j^{\prime}|i,j)= \tilde P^{\pi}((i-1)m_a+j,(i^{\prime}-1)m_a+j^{\prime}).\]


\subsection{Limit Theorems under Sufficient Exploration} \label{sec:stand}
We first establish the asymptotic Normality of $\hat Q_n$ under an exploration policy $\pi$:
\begin{theorem}\label{thm:clt_basic}
	Under Assumptions \ref{assum:ind} and \ref{assum:unique}, if the data are collected according to $\pi$ satisfying Assumption \ref{assum:postive_stat}, 
	$\hat{Q}_n  \rightarrow Q$ almost surely (a.s.) as $n\to\infty$. Moreover,
	\[ \sqrt{n}(\hat{Q}_n - Q) \Rightarrow \mathcal{N}({\bf 0},\Sigma) \mbox{\ \ as\ \  $n\to\infty$},\]
	where 
	\begin{equation}
	\Sigma = (I-\gamma \tilde{P}^{\pi^{*}})^{-1}W^{-1}(D_R + \gamma^2D_Q) ((I-\gamma \tilde{P}^{\pi^{*}})^{-1})^T,\label{formula var}
	\end{equation}
	$W$, $D_R$ and $D_Q$ are $N\times N$ diagonal matrices with
	\begin{equation} \label{eq:prim1}
	\begin{split}
	&W((i-1)m_a+j,(i-1)m_a+j)= w(i,j),\\  
	& D_R((i-1)m_a+j,(i-1)m_a+j)=\sigma^2_{R}(i,j),\\
	& D_Q((i-1)m_a+j,(i-1)m_a+j)=(V^{*})^T\Sigma_{P_{i,j}}V^{*}, 
	\end{split}
	\end{equation}
	where $\Sigma_{P_{i,j}}$ is defined in \eqref{eq:Sigma_P}.
\end{theorem}

In addition to the asymptotic Normality, a key result in Theorem \ref{thm:clt_basic} is an explicit characterization of the asymptotic variance $\Sigma$, which is derived using the delta method \citep{serfling2009approximation}. Intuitively, it is the product of the sensitivities (i.e., gradient) of $Q$ with respect to its parameters and the variances of the parameter estimates. Here, the parameters are $\mu_R$ and $P$ and the corresponding gradients are $(I-\gamma \tilde{P}^{\pi^{*}})^{-1}$ and $(I-\gamma \tilde{P}^{\pi^{*}})^{-1}V^{*}$ respectively. The variances of these parameter estimates (i.e. \eqref{mu estimate} and \eqref{P estimate}) involve $\sigma^2_{R}(i,j)$, $\Sigma_{P_{i,j}}$, and $w(i,j)$ (i.e., the proportion of samples to each state-action pair). 


Using the relations that $V^{*}(s) = \max_{a\in\mathcal{A}} Q(s,a)$ and  $\hat{V}^{*}_n(s) = \max_{a\in\mathcal{A}} \hat{Q}_n(s,a)$, we can leverage Theorem \ref{thm:clt_basic} to further establish the asymptotic Normality of $\hat{V}^{*}_n$ and $\hat{\chi}^{*}_n$:

\begin{corollary}\label{cor:V_R_basic}
	Under Assumptions \ref{assum:ind}, \ref{assum:unique} and \ref{assum:postive_stat},
	\[ \sqrt{n}(\hat{V}^{*}_n -V^{*}) \Rightarrow \mathcal{N}(0,\Sigma_V) \mbox{\ \  and\ \ }
	\sqrt{n}(\hat{\chi}^{*}_n - \chi^{*}) \Rightarrow \mathcal{N}(0,\sigma_{\chi}^2) \mbox{\ \ as\ \  $n\to\infty$}\]
	where \[\Sigma_V=(I-\gamma P^{\pi^{*}})^{-1}(W^{\pi^{*}})^{-1}[D_R^{\pi^{*}} +  \gamma^2D_{V}^{\pi^{*}}]((I-\gamma P^{\pi^{*}})^{-1})^T,\] 
	and
	\[\sigma_{\chi}^2=\rho ^T \Sigma_V \rho,\]
	where $P^{\pi^{*}}$ is an $m_s\times m_s$ transition matrix with $P^{\pi^{*}}(i,j)=P(j|i,a^*(i))$,
	$W^{\pi^*}$, $D_R^{\pi^{*}}$ and $D_{V}^{\pi^{*}}$ are $m_s\times m_s$ diagonal matrices with
	\[W^{\pi^*}(i,i)=w(i,a^*(i)), ~ D_R^{\pi^{*}}(i,i)=\sigma_R^2(i,a^*(i)), ~ \mbox{and} ~ D_{V}^{\pi^{*}}(i,i)=(V^*)^T \Sigma_{P_{i,a^*(i)}}V^*\] 
	respectively.
	
	
	
\end{corollary}


Theorem \ref{thm:clt_basic} and Corollary \ref{cor:V_R_basic} can be used immediately for statistical inference. 
In particular, we can construct confidence regions for subsets of the Q-values jointly,
or confidence intervals for linear combinations of the Q-values.
For example, a quantity of interest that we will later utilize to design good exploration policies is $Q(s,a_1)-Q(s, a_2)$, i.e., the difference between action $a_1$ and $a_2$ when the agent is in state $s$.
Define $\sigma_{\Delta Q}^2$ as
\begin{equation} \label{eq:sigmaD}
\sigma_{\Delta Q}^2(s,a_1,a_2)=\left(e_{(s-1)m_a+a_1}-e_{(s-1)m_a+a_2}\right)^T \Sigma \left(e_{(s-1)m_a+a_1}-e_{(s-1)m_a+a_2}\right)
\end{equation}
and its estimator $\hat\sigma_{\Delta Q,n}^2$ by replacing $Q$, $V^{*}$, $\sigma^2_{R}$, $w$, $P$ with
$\hat{Q}_n$,$\hat{V}^{*}_n$, $\hat\sigma^2_{R,n}$, $\hat w_n$, $\hat{P}_n$ in $\Sigma$.
Then, the $100(1-\alpha)\%$ confidence interval for $Q(s,a_1)-Q(s, a_2)$ takes the form
\[C_n(s,a_1,a_2; \alpha)=\left(\hat Q_n(s,a_1)-\hat Q_n(s, a_2)\right) \pm z_{\alpha/2} \hat\sigma_{\Delta Q,n}(s,a_1,a_2),\]
where $z_{\alpha/2}$ is the $(1-\alpha/2)$-quantile of $\mathcal{N}(0,1)$, i.e., $\mathbb{P}(\mathcal{N}(0,1)\leq z_{\alpha/2})=1-\alpha/2$.
This confidence interval is asymptotically valid in the sense that \[\lim_{n\rightarrow\infty} \mathbb{P}(Q(s,a_1)-Q(s,a_2)\in C_n(s,a_1,a_2;\alpha))=1-\alpha.\]

Following a similar approach as the proof of Theorem \ref{thm:clt_basic}, we can also establish the asymptotic Normality for the estimated value function under any given policy $\tilde{\pi}$. In this case, the value function $V^{\tilde\pi}$ satisfies the equation 
\[V^{\tilde{\pi}}(s) = \sum_a \mu_R(s,a) \tilde{\pi}(a|s) + \gamma \sum_a \tilde{\pi}(a|s) \sum_{s'} P(s'|s,a) V^{\tilde{\pi}}(s').\] 
Denote the estimator of 
$ V^{\tilde{\pi}}$ as $ \hat{V}_n^{\tilde{\pi}}$, which satisfies
\[\hat V_n^{\tilde{\pi}}(s) = \sum_a \hat \mu_{R,n}(s,a) \tilde{\pi}(a|s) + \gamma \sum_a \tilde{\pi}(a|s) \sum_{s'} \hat P_n(s'|s,a) \hat V_n^{\tilde{\pi}}(s').\]

\begin{corollary}\label{cor:V_basic_pi}
	Under Assumptions \ref{assum:ind} and \ref{assum:postive_stat}, 
	\[ \sqrt{n}(\hat{V}^{\tilde{\pi}}_n -V^{\tilde{\pi}}) \Rightarrow \mathcal{N}(0,\Sigma^{\tilde{\pi}}_V) \mbox{\ \ as\ \  $n\to\infty$}\]
	where \[\Sigma^{\tilde{\pi}}_V= (I-\gamma P^{\tilde{\pi}})^{-1}  D^{\tilde{\pi}} \left((I-\gamma P^{\tilde{\pi}})^{-1} \right)^T,\] 
	$P^{\tilde{\pi}}$ is an $m_s\times m_s$ transition matrix with $P^{\tilde{\pi}}(i,j)=\sum_{1\leq a\leq m_a}P(j|i,a) \tilde{\pi}(a|i)$,
	$D^{\tilde{\pi}}$ is an $m_s\times m_s$ diagonal matrix with
	$$D^{\tilde{\pi}}(i,i)= \sum_{1\leq j\leq m_a} \frac{\tilde{\pi}(j|i)^2}{w(i,j)} \left[(\gamma V^{\tilde{\pi}})^T \Sigma_{P_{i,j}}(\gamma V^{\tilde{\pi}}) + \sigma^2_R(i,j) \right].$$
\end{corollary}

 Corollary \ref{cor:V_basic_pi} essentially recovers 
 Corollary 4.1  in  \citet{mannor2007bias}. Different from  \citet{mannor2007bias}, we derive our results by using an implicit function theorem on the corresponding fixed-point equation to obtain the gradient of $Q$, viewing the latter as the solution to the equation and as a function of $\mu_R, P$. This approach is able to generalize the results for a fixed policy in \citet{mannor2007bias} to the optimal value functions, and  provide distributional statements as Theorem \ref{thm:clt_basic} and Corollary \ref{cor:V_R_basic} above. We also note that an alternative route to obtain our results is to conduct perturbation analysis on the linear programming (LP) representation of the MDP, which would also give gradient information of $V^*$ (and hence $Q$ as well). The implicit function approach, nonetheless, is more elementary and intuitive.



\subsection{Non-Unique Optimal Policy} \label{sec:non-unique}
Theorem \ref{thm:clt_basic} above is developed assuming the optimal solution for the MDP is unique. A natural question to ask is what would happen if Assumption \ref{assum:unique} does not hold.
When the optimal solution for the MDP is not unique, the estimated $\hat Q_n$ and $\hat V_n^*$ may ``jump'' around different optimal actions, leading to a more complicated large-sample behavior. We elaborate on this next. Let 
\[U= \{u\in \mathbb{R}^{m_s N+ N}:||u||=1 \}.\]
\begin{theorem} \label{thm:non-unique}
	Suppose Assumptions \ref{assum:ind} and \ref{assum:postive_stat} hold, but the optimal policy is not unique. Then there exists $K\geq1$ distinct $m_s \times (Nm_s+N)$  matrices  $\{G_k\}_{1\leq k \leq K}$ and a deterministic partition of $U$, $\{U_k\}_{1\leq k \leq K}$, i.e., $U= \cup_{1\leq k \leq K}U_k$, such that
\[\sqrt{n}(\hat{V}_n^* - V^{*}) \Rightarrow \sum^K_{k=1} G_k  \mathds{1}\left( Z/\|Z\| \in U_k \right) Z \mbox{\ \ as\ \  $n\to\infty$},\] 
where $Z \sim \mathcal{N}({\bf 0}, \Sigma_{R,P})$,  
\begin{equation} \label{eq:Dp}
	\Sigma_{R,P} = \begin{pmatrix} W^{-1}D_R  & {\bf 0}  \\
	{\bf 0} &  D_P\\
	\end{pmatrix}, 
	~~~ 
	D_P = \begin{pmatrix} \frac{\Sigma_{P_{1,1}}}{w(0m_a+1)}& & & &  \\
	& \ddots & & &\\
	& &  \frac{\Sigma_{P_{i,j}}}{w((i-1)m_a+j)} & &\\
	& & &\ddots  &\\
	&&&& \frac{\Sigma_{P_{m_s,m_a}}}{w((m_s-1)m_a+m_a)}\\
	\end{pmatrix},
\end{equation}
and $W$ and $D_R$ are defined in \eqref{eq:prim1}.
\end{theorem}

To see Theorem \ref{thm:non-unique}, note that the non-uniqueness of the optimal solution can be formalized as a non-degeneracy in the LP representation of the MDP. In particular, when the optimal policy is not unique, the LP representation of the optimal value function has multiple optimal bases. When perturbing $[P,\mu_R]$ along a specific direction $u\in U$ for a small enough amount, at least one of the bases remains optimal. Let $B_u$ denote an optimal basis under the perturbation $u$ and let $\tilde G_{B_u}$ denote the derivative of the optimal solution to the perturbed LP with respect to $[P,\mu_R]$. Then, the corresponding directional derivative takes the form $u^T\tilde G_{B_u}$. Based on the above discussion, $K$ in Theorem \ref{thm:non-unique} is the number of distinct derivatives among the optimal bases, $\{G_k\}_{1\leq k \leq K}$ denotes these derivatives, and $\{U_k\}_{1\leq k \leq K}$ denotes  the corresponding partition of $U$. In particular, for $u\in U_k$, the directional derivative of the optimal value takes the form $u^TG_k$. See Appendix \ref{app:thm2} for more details.

We note from Theorem \ref{thm:non-unique} that if $K=1$, we recover the case of Corollary \ref{cor:V_R_basic}.
However, if $K>1$, the limit distribution becomes non-Gaussian. This arises because the sensitivity to $P$ or $\mu_R$ can be very different depending on the perturbation direction, which is a consequence of non-uniqueness of the solution. 
Lastly, we comment that Theorem \ref{thm:non-unique} can be viewed as an MDP analog to the non-Gaussian limit in sample average approximation where an expected-value optimization may have multiple optimal solutions \citep{shapiro2014lectures}.

\subsection{Constrained Problems}
The MDP introduced in Section \ref{sec:problem} does not have any constraint. In recent years, motivated by budgeted decision-making \citep{boutilier2016budget} and safety-critical applications \citep{achiam2017constrained,chow2017risk}, there are growing interests 
in constrained MDPs.
The goal is to maximize the long-run accumulated discounted reward, $V^{\pi}(s)$,
while ensure that a long-run accumulated discounted cost, denoted as $L^{\pi}(s) =\mathbb{E}[ \sum^{\infty}_{t=0} \gamma^t C(s_t, \pi(s_t)) |s_0 =s]$, is less than some given budget $\eta$, i.e.,
\begin{equation}
\max_{\pi}\sum_s\rho(s)V^{\pi}(s)\text{\ \ subject to\ \ }\sum_s\rho(s)L^\pi(s)\leq\eta\label{CMDP}
\end{equation}
We refer to $L^{\pi}(s)$ as the loss function.
Assume data is generated as before. In addition, we have observations on the incurred cost $\{c_t(s_t,a_t): 1\leq t\leq n\}$. 
Define the sample mean and sample variance of the random costs as 

\begin{equation*}
\hat{\mu}_{C,n}(s=i, a=j)= \begin{cases} \frac{\sum_{1\leq t \leq n} c_t(s_t, a_t)\mathds{1}(s_t=i, a_t=j)}{\sum_{1\leq t\leq n} \mathds{1}(s_t=i, a_t=j)}, \quad \text {if $\sum_{1\leq t\leq n} \mathds{1}(s_t=i, a_t=j)>0$ }\\ 
u_C^0(i,j) \quad \text {if $\sum_{1\leq t\leq n} \mathds{1}(s_t=i, a_t=j)=0$; }\\
\end{cases} 
\end{equation*}
\begin{equation*}
\hat{\sigma}^2_{C,n}(s=i,a=j)= \begin{cases}
\frac{\sum_{1\leq t \leq n} c_t(s_t, a_t)^2\mathds{1}(s_t=i, a_t=j)}{\sum_{1\leq t\leq n} \mathds{1}(s_t=i, a_t=j)}-\hat{\mu}_{C,n}(i, j)^2  \quad \text {if $\sum_{1\leq t\leq n} \mathds{1}(s_t=i, a_t=j)>0$ } \\
s_C^0(i,j) \quad \text {if $\sum_{1\leq t\leq n} \mathds{1}(s_t=i, a_t=j)=0$, }
\end{cases} 
\end{equation*}
where $u_C^0(i,j)\in \mathcal{R}$ and $s_C^0(i,j)\geq 0$ are some appropriate initial values for the mean and variance if no data is available.

We  follow our paradigm to solve the empirical counterpart of the problem, namely to find 
a policy $\hat\pi_n^*$ that solves (\ref{CMDP}) by using $\hat V_n^{\pi}(s)$ and $\hat L^\pi_n(s)$ instead of  $V^{\pi}(s)$ and $L^{\pi}(s)$,  where $\hat V_n^{\pi}(s)$ and $\hat L^\pi_n(s)$ are the value functions and loss functions estimates using the empirical $\hat\mu_{R,n}$, $\hat\mu_{C,n}$, $\hat P_n$. We focus on the estimation error of the optimal values (instead of the feasibility, which could also be important but not pursued here).

To understand the error, we first utilize an optimality characterization of constrained MDPs. In general, an optimal policy for \eqref{CMDP}  
is a ``split'' policy \citep{feinberg2012splitting}, namely, a policy that is deterministic except at one state where a randomization between two different actions is allowed. This characterization can be deduced from the associated LP using the occupancy measure \citep{altman1999constrained}. We refer to the randomization probability as the mixing parameter and denote the optimal mixing parameter as $\alpha^*$. In particular, whenever this state, say $s_r$, is visited, action $a_1^*(s_r)$ is chosen with probability $\alpha^*$  and action $a_2^*(s_r)$ is chosen with probability $1-\alpha^*$. 



\begin{theorem}\label{thm CMDP}
	Suppose Assumptions \ref{assum:ind} and \ref{assum:postive_stat} hold and there is a unique optimal policy. Moreover, assume that there is no deterministic policy $\pi$ that satisfies $\sum_s\rho(s)L^\pi(s)=\eta$. Then
	\[\sqrt n(\hat V_n^*-V^*)\Rightarrow N(0,\Sigma_c) \mbox{ as $n\to \infty$},\] 
	where one of the following two cases hold:
	
	\noindent{\bf 1.} The optimal policy is deterministic. In this case,
	$\Sigma_c=\Sigma_V$ where $\Sigma_V$ is defined in Corollary \ref{cor:V_R_basic}.
	
	\noindent{\bf 2.} The optimal policy is deterministic, except at one state, $s_r$, where a randomization between two actions, $a^{*}_1(s_r)$ and $a^{*}_2(s_r)$, occurs, with the mixing parameter $\alpha^*$. 
	In this case,
	\begin{align*}
	&\Sigma_c=\left((I-\gamma P^{\pi^{*}})^{-1}[G^{\pi^{*}}, 0,H^{\pi^{*}}_V]-\frac{(I-\gamma P^{\pi^{*}})^{-1} h_V \rho^T (I-\gamma P^{\pi^{*}})^{-1}[0,G^{\pi^{*}},H^{\pi^{*}}_L]}{\rho^T (I-\gamma P^{\pi^{*}})^{-1} h_L } \right) \\
	&\quad \times \Sigma_{R,C,P} \left((I-\gamma P^{\pi^{*}})^{-1}[G^{\pi^{*}}, 0,H^{\pi^{*}}_V]-\frac{(I-\gamma P^{\pi^{*}})^{-1} h_V \rho^T (I-\gamma P^{\pi^{*}})^{-1}[0,G^{\pi^{*}},H^{\pi^{*}}_L]}{\rho^T (I-\gamma P^{\pi^{*}})^{-1} h_L } \right) ^T,
	\end{align*}
	where 
	\begin{align*}
	\Sigma_{R,C,P} = \begin{pmatrix} W^{-1}D_R  & {\bf 0} & {\bf 0}   \\
	{\bf 0} & W^{-1}D_C & {\bf 0}\\
	{\bf 0} & {\bf 0} &  D_P\\
	\end{pmatrix}, 
	\end{align*}
	$D_C$ is an $N\times N$ diagonal matrix with $D_C((i-1)m_a+j,(i-1)m_a+j)=\sigma^2_{C}(i,j)$, 
	$W$ and $D_R$ are defined in \eqref{eq:prim1}, and $D_P$ is defined in \eqref{eq:Dp}.
	$h_V$ and $h_L$ are $m_s$-dimensional vectors. When $s=s_r$
	\[h_V(s) = (\mu_R(s, a_1^{*}(s)) - \mu_R(s, a_2 ^{*}(s)) + \sum^{m_s}_{j=1} \gamma V^{\pi^{*}}(j)(P(j|s,a_1^{*}(s))-P(j|s,a_2^{*}(s)) ),\] 
	\[h_L(s) = (\mu_C(s, a_1^{*}(s)) - \mu_C(s, a_2 ^{*}(s)) + \sum^{m_s}_{j=1} \gamma L^{\pi^{*}}(j)(P(j|s,a_1^{*}(s))-P(j|s,a_2^{*}(s)) ).\]
	when $s\neq s_r$, $q_V(s)=q_L(s)=0$.
	\begin{align*}
	G^{\pi{*}} = \begin{pmatrix} \pi^{*}(\cdot|1)^T  & &  \\
	& \ddots & \\
	& &  \pi^{*}(\cdot|m_s)^T \\
	\end{pmatrix} ,
	H^{\pi{*}}_V = \begin{pmatrix} q_V^{\pi^{*}}(1)^T & &  \\
	& \ddots & \\
	& & q_V^{\pi^{*}}(m_s)^T \\
	\end{pmatrix},
	H^{\pi{*}}_L = \begin{pmatrix} q_L^{\pi^{*}}(1)^T & &  \\
	& \ddots & \\
	& & q_L^{\pi^{*}}(m_s)^T \\
	\end{pmatrix},
	\end{align*}
	where $q_V^{\pi^{*}}(i)$ and $q_L^{\pi^{*}}(i)$ are $N$-dimensional row vectors:
	\[q_V^{\pi^{*}}(i) = \gamma \left[\pi^{*}(1|i)(V^{\pi^{*}})^T,\dots, \pi^{*}(j|i)(V^{\pi^{*}})^T,\dots, \pi^{*}(m_a|i)(V^{\pi^{*}})^T\right]^T,\]
	\[q_L^{\pi^{*}}(i) = \gamma \left[\pi^{*}(1|i)(L^{\pi^{*}})^T,\dots, \pi^{*}(j|i)(L^{\pi^{*}})^T,\dots, \pi^{*}(m_a|i)(L^{\pi^{*}})^T\right]^T.\] 
\end{theorem}

Case 1 in Theorem \ref{thm CMDP}  corresponds to the case where the constraint in \eqref{CMDP} is not binding. This effectively reduces the constrained MDP to the unconstrained scenario in Corollary \ref{cor:V_R_basic}, since a small perturbation of $\mu_R,\mu_C,P$ does not affect feasibility. Case 2 is when the constraint is binding. In this case, $\alpha^*$ must be chosen such that the split policy ensures equality in the constraint. When $\mu_R,\mu_C,P$ are perturbed, the estimated $\hat\alpha_n^*$ would adjust accordingly. In this case, $\hat V_n^*$ incurs two sources of noises, one from the uncertainty in $\hat\mu_{R,n},\hat P_n$ that appears also in unconstrained problems, and the other from the uncertainty in calibrating $\hat\alpha_n^*$ that is affected by $\hat\mu_{C,n},\hat P_n$. This latter source of noise leads to the extra terms in the asymptotic variance expression.

We assume in Theorem \ref{thm CMDP} that there is no deterministic policy $\pi$ satisfying $\sum_s\rho(s)L^\pi(s)=\eta$. This requirement is imposed to ensure that the corresponding LP formulation of the constrained MDP is non-degenerate. Similar to Assumption \ref{assum:unique}, when it does not hold, the limiting distribution can be non-Gaussian (see, e.g., Theorem \ref{thm:non-unique}).

\section{Exploration Policy}\label{sec:explore}
In this section, we utilize our results in Section \ref{sec:ucq} to design good exploration policies. We focus on the setting in which an agent is assigned a period to collect data by running the state transition with an exploration policy. The goal is to obtain the best policy at the end of the period in a probabilistic sense, i.e., minimize the probability of selecting a suboptimal policy. We restrict our development to the standard case where each state has a unique optimal action as discussed in Section \ref{sec:stand}.

We first define, for $i\in\mathcal{S}$, $j\in \mathcal{A}$ and $j\neq a^*(i)$, the relative discrepancy as
\begin{equation}\label{eq:hij}
h_{ij}=\left(Q(i,a^*(i))-Q(i,j)\right)^2/\sigma_{\Delta Q}^2(i,a^*(i),j),
\end{equation}
where $\sigma^2_{\Delta Q}(i,a^*(i),j)$ is defined in \eqref{eq:sigmaD}. 
Intuitively, $h_{ij}$ captures the relative ``difficulty'' in obtaining the optimal policy given the estimation errors of Q-values. 
In particular, if the Q-values are far apart, or if the estimation variance is small, then $h_{ij}$ is large which signifies an ``easy'' problem, and vice versa. 

Our proposed strategy attempts to maximize the worst of $h_{ij}$'s, i.e.,
\begin{equation}\label{eq:opt2}
\max_{w\in\mathcal{W}} \min_{i\in \mathcal{S}}\min_{j\in \mathcal{A},j\neq a^*(i)}h_{ij},
\end{equation}
where $w$ denotes the proportion of visits to each state-action pair, within some allocation set $\mathcal{W}$ which will be defined in \eqref{eq:W} below. 
Based on our interpretation of $h_{ij}$'s,
criterion \eqref{eq:opt2} aims to make the problem the ``easiest'' to differentiate the optimal policy. Alternatively, one can also interpret \eqref{eq:opt2} from a large deviation point of view \citep{glynn2004large,dong2016three}. Suppose the Q-values for state $i$ between two different actions $a^*(i)$ and $j$ are very close. Then, one can show that the probability of a suboptimal selection of $j$ has roughly an exponential in $n$ decay rate controlled by $h_{ij}$. Obviously, there can be many more comparisons to consider, but the exponential form dictates that the smallest decay rate dominates the calculation, thus leading to the inner minimizations in \eqref{eq:opt2}. Criterion like \eqref{eq:opt2} is motivated by the OCBA procedure in simulation optimization, which in general considers simple mean-value alternatives \citep{chen2011stochastic}. Here, we consider the estimation of Q-values. Thus, we refer to our procedure as Q-OCBA.

Implementing criterion \eqref{eq:opt2} requires addressing two additional considerations. First, solving \eqref{eq:opt2} needs the model primitives $Q$, $P$, and $\sigma^2_R$, which appear in the expression of $h_{ij}$. These quantities are unknown a priori, but as we collect more data, they can be sequentially estimated. This leads to a multi-stage parameter update plus optimization scheme. Second, since data are collected through running a Markov chain, not all allocation $w$ is \emph{admissible}, i.e., realizable as the stationary distribution of some Markov chain. To resolve this issue, we next derive a convenient characterization for admissibility. 

Accordingly to Assumption \ref{assum:postive_stat}, we call a policy $\pi(\cdot|s)$ admissible if the Markov Chain with transition probability $\tilde{P}^{\pi}$ is irreducible, and we denote $w_{\pi}$ as its stationary distribution. Define 
\begin{equation}\label{eq:W}
\begin{split}
\mathcal{W} = &\left\{w > 0: \sum_{1\leq j \leq m_a} w((i-1)m_a + j) =
\sum_{1\leq k \leq m_s} \sum_{1\leq l \leq m_a} w((k-1)m_a+l) P(i|k,l)\right.\\
&\left. \quad \forall 1 \leq i \leq m_s,
\sum_{1\leq i \leq m_s} \sum_{1\leq j \leq m_a} w((i-1)m_a+j) = 1 \right\}.
\end{split}\end{equation} 
The following lemma provides a characterization for the set of admissible policies.
\begin{lemma} \label{lemma:feasible set}
	For any admission policy $\pi$, $w_{\pi}\in \mathcal{W}$. For any $w\in \mathcal{W}$, $\pi_w$ with $\pi_w(a=j|s=i)  = w((i-1)m_a + j)/ \left(\sum^{m_a}_{k=1} w((i-1)m_a + k)\right)$ is an admissible policy.
\end{lemma}



Lemma \ref{lemma:feasible set} implies that optimizing over the set of admissible policies is equivalent to optimizing over the set of stationary distributions. The latter is much more tractable thanks to the linear structure of $\mathcal W$. In practice, we use $\mathcal{W}_\eta = \mathcal{W} \cap \{w(i) \geq \eta, i=1,\dots,N\}$ for some small $\eta$ to ensure closedness of the set (In our numerical experiments, we use $\eta=10^{-6}$). 

To elaborate the criterion \eqref{eq:opt2}, we note that it can be equivalently written as $\min_{w\in\mathcal{W}} \max_{i\in \mathcal{S}}\max_{j\in \mathcal{A},j\neq a^*(i)}1/h_{ij}$. Then, plugging the formula for $\sigma_{\Delta Q}^2(i,a^*(i),j)$ (i.e., \eqref{eq:sigmaD}) in $h_{ij}$, we have 
\begin{equation}\label{eq:opt3}
\min_{w\in\mathcal{W}} \max_{i\in \mathcal{S}}\max_{j\in \mathcal{A},j\neq a^*(i)} \sum_{(s,a)\in\mathcal{S}\times \mathcal{A}} \frac{c_{ij}(s,a)}{w_{s,a}},
\end{equation} 
where 
\begin{equation}\label{eq:cij}
c_{ij}(s,a)=\frac{(H_{ij}((s-1)m_a+a))^2\left(\sigma_R^2(s,a)+(V^*)^T\Sigma_{P_{s,a}}V^*\right)}{\left(Q(i,a^*(i))-Q(i,j)\right)^2}
\end{equation}
and
\[
H_{ij}=\left(e_{(i-1)m_a+a^*(i)}-e_{(i-1)m_a+j}\right)^T(I-\gamma \tilde{P}^{\pi^{*}})^{-1}.
\]
Note that $c_{ij}(s,a)$ can be easily estimated with plug-in estimators. 
Next, note that the objective function in \eqref{eq:opt3}
is convex with $N$ variables and the feasible region $\mathcal{W}_{\eta}$ can be characterized by $m_s+1+N$ linear constraints. Thus, \eqref{eq:opt3} can be solved by standard convex optimization solvers efficiently (in our numerical experiments in Section \ref{sec:numerics}, solving it takes less than 0.01 seconds).

We are now ready to introduce our Q-OCBA implementation -- Algorithm \ref{alg:qocba}.





\begin{algorithm}[h]
	\textbf{Input:} Number of iterations $K$, 
	cumulative number of data collected after stage $k$, $B_k$, $k=1,\dots, K$, with $B_K=n$, i.e., $n$ is the total sampling budget (and we define $B_{0}=0$), auxiliary parameters $\eta>0$ (a small constant), $C_u>0$ (a large constant), and $C_l>0$ (a small constant),
	and initial exploration  policy $\pi_1$. \\
	\textbf{Initialization:} $k=1$\;
	\While{$k\leq K$}{
		Run $\pi_k$ for $B_{k} -B_{k-1}$ steps and collect the corresponding data;
		Calculate $\hat{P}_{B_k}$, $\hat\mu_{R,B_k}$, and $\hat{\sigma}^2_{R,B_k}$ based on the $B_k$ data points collected;\\
		Apply value-iteration using $\hat{P}_{B_k}$ and $\hat{\mu}_{R,B_k}$ to obtain $\hat{Q}_{B_k}$; \\
		Solve the following adaption of \eqref{eq:opt3} for the optimal $w_k$
        \begin{equation}\label{eq:optq}
        \min_{w\in\mathcal{W}_{\eta}} \max_{i\in \mathcal{S}}\max_{j\in \mathcal{A},j\neq a^*(i)} \sum_{(s,a)\in\mathcal{S}\times \mathcal{A}}\frac{\max\left(C_l, \min\left(C_u, \hat c_{ij,B_k}(s,a)\right)\right)}{w_{s,a}},
        \end{equation}
        where $\hat c_{ij,B_k}(s,a)$ denotes the plug-in estimator of $c_{ij}(s,a)$ defined in \eqref{eq:cij};\\
		Set $\pi_{k+1}(a=j|s=i)  = w_k((i-1)m_a + j)/ \sum^{m_a}_{l=1} w_k((i-1)m_a + l)$ and $k=k+1$\;		
	}
	\caption{Q-OCBA sequential updating rule for exploration} \label{alg:qocba}
\end{algorithm}

 Let $\pi^{e}$ denote the optimal exploration policy (here $e$ stands for exploration) and $w^*$ denote the corresponding stationary distribution, i.e., $w^*$ solves \eqref{eq:opt3}. Note that $\pi^e$ is different from $\pi^*$. The former is used to collect data, while the latter is used to collect reward in the deployment stage and is to be learned offline using the data collected under $\pi^e$. We also denote $\hat{\pi}_n$ as the exploration policy derived from Algorithm \ref{alg:qocba} after collecting $n$ data points. The following theorem establishes the asymptotic consistency of the Q-OCBA sequential updating rule.
\begin{theorem} \label{thm:Q-OCBA_convergence}
Under Assumptions \ref{assum:ind}, \ref{assum:unique}, and \ref{assum:postive_stat}, 
for Q-OCBA defined in Algorithm \ref{alg:qocba}, suppose one of the following two cases holds as we send the sampling budget $n=B_K$ to infinity: 
\begin{enumerate}
    \item The number of iterations $K$ is fixed. This implies that there exists at least one $\hat{k}\in\{1,\dots,K\}$, such that $B_{\hat{k}} - B_{\hat{k}-1}\rightarrow\infty$ as $n\rightarrow\infty$.
    \item The number of iterations $K\rightarrow\infty$ as $n\rightarrow \infty$ and 
    $B_k - B_{k-1} \geq N$ for all $k=1,2,\dots$. 
\end{enumerate}
Then, 
\[\hat{Q}_n\rightarrow Q \mbox{ and } 
\hat{\pi}_n\rightarrow \pi^{e} \mbox{ a.s. as $n\rightarrow\infty$}.\]
Moreover, in Case 2, if there exists a constant $\gamma>0$ such that $B_k - B_{k-1} \geq \gamma k$, $k=1,2,\dots$, then
\[\frac{1}{n}\sum_{t=1}^{n}\mathds{1}(s_t=s, a_t=a) \rightarrow w^*(s,a) \mbox{ a.s. for any $(s,a)\in\mathcal{S}\times\mathcal{A}$ as $n\rightarrow\infty$.}\]
\end{theorem}

The first result in Theorem \ref{thm:Q-OCBA_convergence} indicates that as we collect more data in Q-OCBA, our estimated $Q$-values will get closer to the true $Q$-values and our estimated exploration policy will get closer to the optimal exploration policy. This is achieved by showing that each state-action pair is visited infinitely often under our policy. The second result shows that if the number of iterations is increasing with the sampling budget and the batch size $B_k-B_{k-1}$ is properly increasing with $k$, the overall proportion of samples at each state-action pair is also getting closer to the optimal proportion as we collect more data. Recall that $w_k$ denotes the stationary distribution under policy $\pi_k$ in Algorithm \ref{alg:qocba}. From the first result, in Case 2, we have $w_k\rightarrow w^*$ as $k\rightarrow\infty$. Thus, for the second result, we only need to show that the convergence happens sufficiently fast. To characterize the convergence rate, we employ the large deviations theory for the empirical measure of the underlying Markov chain.

Theorem \ref{thm:Q-OCBA_convergence} lays out fairly general conditions for $K$ and $B_k$'s, under which Algorithm \ref{alg:qocba} achieves consistency. However, with a finite budget in practice, we need to be more mindful about the choice of these parameters. In Section \ref{sec:sen}, we numerically investigate  how to tune Q-OCBA. In general, we suggest setting $K$ between $6$ to $10$. As for $B_k$'s, when $n$ or $K$ is small, we suggest evenly splitting the sampling budget among the $K$ stages, i.e., $B_k-B_{k-1}= n/K$. When $n$ and $K$ are large, we can use gradually increasing batch sizes, e.g., $B_k-B_{k-1}=\gamma k$ where $\gamma=2n/(K(K+1))$.

\section{Approximations for Large State Space and/or Action Space} \label{sec:approx}
In many applications, the state space $\mathcal{S}$ and/or the action space $\mathcal{A}$ can be very large. In those settings, updating an $m_s \times m_a$ look-up table via $\mathcal{T}_{\mu_R, P}$ can be computationally infeasible, and approximation/parameterization is often employed. In this section, we study two approximation schemes: approximate value iteration and kernel representation. 

\subsection{Approximate Value Iteration}
When $m_s$ is large, approximate value iteration can be employed. It operates by applying a mapping $M$  over $\mathcal{T}_{\mu_R,P}$. In many cases, $M = M_g \circ M^{\mathcal{S}_0}_I$, where $M^{\mathcal{S}_0}_I$ is a dimension-reducing ``inherit'' mapping $\mathbb{R}^{m_s m_a} \rightarrow \mathbb{R}^{m_{s_0} m_a}$ and $M_g$ is the ``generalization'' mapping $\mathbb{R}^{m_{s_0} m_a} \rightarrow \mathbb{R}^{m_{s} m_a}$ that lifts back to the full dimension. 
By selecting a ``representative'' subset $\mathcal{S}_0 \subset \mathcal{S}$ with cardinality $m_{s_0} \ll m_s$, $M^{\mathcal{S}_0}_I$ is defined as 
$M^{\mathcal{S}_0}_I(x) = [x(i,j)]_{i\in \mathcal{S}_0, 1\leq j \leq m_a}$ where  $[x_i]_{i\in I}$ denotes the set of $x_i$'s whose index $i\in I$. 
The idea is that we update the Q-values for $s\in \mathcal{S}_0$ only, and use the mapping $M_g$ to extrapolate the other Q-values.
In this setup, we define $Q^{M}$ as a fixed point of the operator 
$M \circ \mathcal{T}_{\mu_R, P}(\cdot)$, 
and $V^M(s)=\max_{a} Q^M(s,a)$. We also define $Q^M_{\mathcal{S}_0}=M^{\mathcal{S}_0}_I\circ\mathcal{T}_{\mu_R, P}(Q^M)$ as the dimension-reduced Q-values. 




We next study the large-sample behavior of an approximate value iteration scheme. 
We first introduce a few assumptions.
Assume $\mathcal{S}$ is a complete vector space.
Define the max norm as 
$\|a\|_{\infty}=\max_{i}|a_i|$.
To guarantee the existence of $Q^M$, we make the following assumption on the generalization map $M_g$:
\begin{assumption}\label{ass:large_state_unique}
	$M_g$ is a max-norm non-expansion mapping in $\mathcal S$, i.e., 
	\[\|M_g(x)-M_g(y)\|_{\infty} <\| x-y\|_{\infty} \mbox{ for any $x, y\in \mathcal{S}$.}\]
\end{assumption}

Under Assumption \ref{ass:large_state_unique}, $M \circ \mathcal{T}_{\mu_R, P}(\cdot)$ is still a contraction mapping (Theorem 3.1 in \citealt{gordon1995stable}). Thus, $Q^M$ is well defined. On the other hand, when this assumption is not satisfied, there exist MDPs for which $M \circ \mathcal{T}_{\mu_R, P}(\cdot)$ does not have a unique fixed point \citep{gordon1995stable}.
Assumption \ref{ass:large_state_unique} is generally satisfied by ``local'' approximation methods such as linear interpolation, $k$-nearest neighbors, and local weighted averaging. Moreover, \cite{gordon1995stable} defines the notion of an averager and shows that any $M_g$ associated with an averager is max-norm non-expansion (Theorem 3.2 in \citealt{gordon1995stable}). In particular, $M_g$ is an averager if the fitted values (i.e., $Q(s,a)$ for $s\in \mathcal{S}$) are the weighted average of some target values (i.e., $Q(s,a)$ for $s\in \mathcal{S}_0$) and some predetermined constants.


We also need the following analogs of Assumptions \ref{assum:unique} and \ref{assum:postive_stat} to $Q^M$ and $\mathcal S_0$:
\begin{assumption}\label{ass:large_state_unique2}
	For any state $s\in \mathcal{S}$, $\arg\max_{a\in \mathcal{A}} Q^M(s,a)$ is unique.
\end{assumption}
\begin{assumption}\label{ass:large_state_stationary}
	For the Markov Chain with transition probability $\tilde P^{\pi}$, the set of states $\{(s,a): s\in\mathcal{S}_0, a\in\mathcal{A}\}$ is irreducible.
\end{assumption}
Let $N_0=m_{s_0}m_a$ and $I_{\mathcal{S}_0}=\{(i-1)m_a+j: i\in\mathcal{S}_0, j\in \mathcal{A}\}$. 
With Assumption \ref{ass:large_state_stationary}, 
we denote $\tilde{P}_{\mathcal{S}_0}^M$ as a sub-matrix of $\tilde P^{\pi}$ that only contains rows with indices in $I_{\mathcal{S}_0}$.
We also denote $\mathcal{S}_0(i)$ as the $i$-th element (state) in $\mathcal{S}_0$. 
Define $\hat{Q}^M_n$ as the empirical estimate of $Q^M$ built on $n$ observations. 

\begin{theorem}\label{thm:clt_large}
	Under Assumptions \ref{ass:large_state_unique}, \ref{ass:large_state_unique2}, and \ref{ass:large_state_stationary},
	if $M_g$ is continuously differentiable, then  
	\[\sqrt{n}(\hat{Q}^M_n - Q^M) \Rightarrow \mathcal{N}(0,\Sigma^M_{\mathcal{S}_0}) \mbox{ as $n\to \infty$},\]
	where
	\begin{align*}
	\Sigma^M_{S_0} =& (I-\gamma \nabla M_g(Q^M_{\mathcal{S}_0}) \tilde{P}_{\mathcal{S}_0}^M)^{-1}
	\nabla M_g(Q^M_{\mathcal{S}_0}) (W^{\mathcal{S}_0})^{-1}[D_R^{\mathcal{S}_0} +  \gamma^2D_Q^{\mathcal{S}_0}]\\
	&\nabla M_g(Q^M_{\mathcal{S}_0})^T((I-\gamma \nabla M_g(Q^M_{\mathcal{S}_0}) \tilde{P}_{\mathcal{S}_0}^M)^{-1})^T,
	\end{align*}
	$\nabla M_g$ is the Jacobian of the mapping $M_g$,
	and $W^{\mathcal{S}_0}$, $D_R^{\mathcal{S}_0}$, $D_Q^{\mathcal{S}_0}$ are $N_0 \times N_0$ diagonal matrices with 
	\[\begin{split}
	 		&W^{\mathcal{S}_0}((i-1)m_a+j,(i-1)m_a+j)=w(\mathcal{S}_0(i),j),\\
	 		&D_R^{\mathcal{S}_0}((i-1)m_a+j,(i-1)m_a+j)=\sigma^2_{R}(\mathcal{S}_0(i),j),\\
	 		&D_Q^{\mathcal{S}_0}((i-1)m_a+j,(i-1)m_a+j)=(V^M)^T\Sigma_{P_{\mathcal{S}_0(i),j}}V^M.
	\end{split}\]
\end{theorem}

An important class of approximation scheme that satisfies Assumption \ref{ass:large_state_unique} is the nearest neighbors \citep{shah2018q}. Given data $q\in\mathbb{R}^{m_{s_0}m_a}$, the nearest-neighbor averaging can be expressed as
\begin{equation}\label{eq:NN}
M_g(q)(s,a)=\sum_{i=1}^{m_{s_0}}K(s,\mathcal{S}_0(i))q(\mathcal{S}_0(i),a), ~ \forall (s,a)\in \mathcal{S}\times\mathcal{A},
\end{equation}
where 
$K(\cdot,\cdot)\geq 0$ is the weighting kernel function satisfying $\sum_{i=1}^{m_{s_0}}K(s,\mathcal{S}_0(i))=1$ for any $s\in \mathcal{S}$ and $K(x,y)=0$ if $\|x-y\|>h$, where $h$ is the width parameter. 
Then, for any $s\in\mathcal{S}$,
$Q^M(s,a)=M_g(Q^M_{\mathcal{S}_0})$;
and for any $s_0\in \mathcal{S}_0$,
\[
Q_{\mathcal{S}_0}^M(s_0,a)
= \mu_R(s_0,a)+\gamma \mathbb{E}\left[\max_{a^{\prime}\in\mathcal{A}}M_g(Q_{\mathcal{S}_0}^M)(s^{\prime},a^{\prime})|s_0,a\right].
\]
We note from \eqref{eq:NN} that the nearest-neighbor averaging is an averager and is thus max-norm non-expansion. 



\subsection{Kernel Representation}
When $N=m_sm_a$ is large, a common practice is to parameterize the high-dimensional Q-values or policy functions using a set features. In this subsection, we study one particular parameterization scheme -- the kernel representation \citep{yang2020reinforcement},
where we assume the transition probability $P$ can be fully embedded in a kernel space:
\begin{assumption} \label{assum:P_kernel}
For each $(s,a) \in \mathcal{S} \times \mathcal{A}$ and $s' \in \mathcal{S}$,
there exists a transition core matrix $M^*\in {R}^{k_1}\times {R}^{k_2}$, such that $$ P(s'|s, a) = \phi(s,a)^T M^* \psi(s')$$ where $\phi(s,a)\in \mathcal{R}^{k_1}$ and $\psi(s^{\prime})\in \mathcal{R}^{k_2}$ are feature functions. 
We assume that $\phi(s,a)$ and $\psi(s^{\prime})$ are known, while $M^*$ has to be estimated from data.
\end{assumption}
$k_1$ and $k_2$ in Assumption \ref{assum:P_kernel} are typically much smaller than $N$. Thus, we can approximate $P$ by a lower dimensional manifold. Similar assumptions are made for the distribution of the reward:
\begin{assumption} \label{assum:R_kernel}
For each $(s,a) \in \mathcal{S} \times \mathcal{A}$, 
there exist parameters $\theta_{\mu} \in {R}^{d_1} $ and $\theta_{\sigma} \in {R}^{d_2}$ , such that $\mathbb{E}[R(s,a)] = \phi_{\mu}(s,a)^T \theta_{\mu} $ and $\mathbb{E}[R^2(s,a)] = \phi_{\sigma}(s,a)^T \theta_{\sigma} $ where $\phi_{\mu}(s,a)\in \mathcal{R}^{d_1}$ and $\phi_{\sigma}(s,a) \in \mathcal{R}^{d_2}$ are feature functions.
We assume that $\phi_\mu(s,a)$ and $\phi_\sigma(s,a)$ are known, while $\theta_\mu$ and $\theta_{\sigma}$ have to be estimated from data.
\end{assumption} 
 
We next define a few more notations to facilitate our subsequent discussion.
Define $M^{*}_v$ as a vector by concatenating columns of $M^{*}$, i.e, $M^{*}_v= [(M^*)_1^{T}, \dots, (M^{*})^{T}_{k_2}]^T$, where $(M^{*})_i$ is the $i$-th column of $M^*$. Then, $ P(s^{\prime}|s, a)$ can be written as
\[\begin{split}
P(s^{\prime}|s, a) 
&= \left(\begin{pmatrix} \phi(s,a)^T  & &  \\
	& \ddots & \\
	& & \phi(s,a)^T \\
	\end{pmatrix} M^{*}_v\right)^T \psi(s^{\prime}) \\
&= \psi(s^{\prime})^T\begin{pmatrix} \phi(s,a)^T  & &  \\
	& \ddots & \\
	& & \phi(s,a)^T \\
	\end{pmatrix} M^{*}_v  
=: F_{\phi, \psi}(s,a,s^{\prime}) M^{*}_v, 
\end{split}\] 
where 
\[F_{\phi, \psi}(s,a,s^{\prime}) =\psi(s^{\prime})^T\begin{pmatrix} \phi(s,a)^T  & &  \\
	& \ddots & \\
	& & \phi(s,a)^T \\
	\end{pmatrix}.
\]
When $n$ data points are available, the least-squares estimation of $M^*$ takes the form: 
\[\hat M_n = A_n^{-1} \sum_{t=1} ^{n} \phi(s_{t}, a_{t}) \psi(s^{\prime}_t)^T K_{\psi}^{-1},\] 
where $K_{\psi} = \sum_{s} \psi(s) \psi(s)^T$  and $A_n = I+ \sum_{t=1} ^{n} \phi(s_{t}, a_{t})\phi(s_{t}, a_{t})^T$.
Define the vector $\hat M_{n, v}$ by concatenating columns of  $\hat M_{n}$, i.e,  $\hat M_{n, v} = [(\hat M_n)_1^T, \dots, (\hat M_n)_{k_2}^T ]^T$, where $(\hat M_n)_i$ is the $i$-th column of $\hat M_n$. Similar to $P(s'|s,a)$, we can define
\begin{equation}\label{eq:PK}
\hat P_{n}^K(s'|s,a) =  F_{\phi, \psi}(s,a,s') \hat M_{n,v}.
\end{equation}

Next, consider an extended Markov chain $\bar X = (s, a, s^{\prime})$ with transition kernel
$$
\bar P^{\pi}((s_2, a_2, s_2^{\prime})|(s_1, a_1, s_1^{\prime})) = 
\begin{cases}
\pi(a_2|s_2) P(s_2^{\prime}|s_2, a_2), & \text {if $s_1^{\prime} = s_2$;}\\
0, & \text {otherwise}.
\end{cases}$$ 
Let $\bar{\mathcal{S}}= \{x_1, x_2, \dots, x_{m_s^2 m_a} \}$ denote the state space of $\bar X$. The states are indexed in the same order as they appear in the rows of $\bar P^{\pi}$. For any state $x$ in $\bar{\mathcal{S}}$, denote its original $(s, a, s')$ representation as $(s(x), a(x), s'(x))$. Let $\bar\omega$ denote the stationary distribution of $\bar X$. Then, $\bar\omega_x = w_{s(x),a(x)} P(s^{\prime}(x)|s(x),a(x))$.
Lastly, define 
\[G(x) = \phi(s(x), a(x)) \psi^T(s^{\prime}(x)) K_{\psi}^{-1} - \phi(s(x), a(x))\phi(s(x), a(x))^T M^*\]
and the vector $G_{v}(x)$ by concatenating columns of  $G(x)$,
i.e., $G_v(x) =[(G(x))^T_1, \dots, (G(x))^T_{k_2}]^T$, where $(G(x))_i$ is $i$-th column of $G(x)$.

\begin{lemma} \label{thm:CLT_kernel_P}
Assume $\mathbb{E}_{\bar\omega}[G^2_v(\bar X)] <\infty$. Then, for any $(s,a,s^{\prime})\in \bar{\mathcal{S}}$,
\[\sqrt{n} (\hat P_{n}^K(s'|s,a) - P(s'|s,a)) \Rightarrow \mathcal{N}({\bf 0}, F_{\phi, \psi}(s,a,s') \Upsilon_E^{-1}\Sigma_{G}\Upsilon_E^{-1}  (F_{\phi, \psi}(s,a,s'))^T), \]
 where 
\[\Upsilon_E = \begin{pmatrix}  I+\sum_{s,a} w_{s,a} \phi(s, a)\phi(s, a)^T  & &  \\
& \ddots & \\
& &  I+\sum_{s,a} w_{s,a} \phi(s, a)\phi(s, a)^T \\
\end{pmatrix} \]
and 
\[\Sigma_{G} = \Var_{\bar\omega} (G_v(\bar X_0)) + 2\sum^{\infty}_{t=1} \Cov_{\bar\omega}(G_v(\bar X_0), G_v(\bar X_t)).\]
Here, $\Var_{\bar\omega}(G_v(\bar X_0))$ is the covariance matrix of $G_v(\bar X_0)$ where $\bar X_0$ follows the stationary distribution $\bar\omega$ and  $\Cov_{\bar\omega}(G_v(\bar X_0), G_v(\bar X_i))$ is a matrix with $\{\Cov_{\bar\omega}(G_v(\bar X_0), G_v(\bar X_t))\}_{s,r}= \Cov_{\bar\omega}((G_v(\bar X_0))_s, (G_v(\bar X_t))_r)$.
Specifically,
\[ 
\Var_{\bar\omega} (G_v(\bar X_0)) = \sum_{(s,a,s^{\prime})\in \bar\S} w_{s,a} P(s^{\prime}|s,a) G_v(s,a,s^{\prime})  G_v(s,a,s^{\prime})^T  
\]
and
\[\begin{split}
&\sum^{\infty}_{t=1} \Cov_{\bar\omega}(G_v(\bar X_0), G_v(\bar X_t))\\
=& \sum_{x_k,x_l\in \bar\S}  w_{s(x_k),a(x_k)} P(s^{\prime}(x_k)|s(x_k),a(x_k)) G_v(x_k)G_v(x_l)^T \left\{(I-\bar P^{\pi})^{-1}\right\}_{k,l}.
\end{split}\]
\end{lemma}

Note that the asymptotic variance $F_{\phi, \psi}(s,a,s') \Psi_E^{-1}\Sigma_{G}\Psi_E^{-1}  (F_{\phi, \psi}(s,a,s'))^T$ can be estimated using plug-in estimators. The sampling policy will affect $\Sigma_G$ through the transition matrix $\bar P^{\pi}$ and the stationary distribution $\bar\omega$.

We next study the estimation of the mean reward $\mu_R$. 
Let $\Phi_{\mu, n}$ denote the observed feature  matrix for the reward, i.e., the $t$-th row of $\Phi_{\mu, n}$ is $\phi_{\mu}(s_t,a_t)^T$, $t=1,\dots,n$. Similarly, let $\Phi_{\sigma, n}$ denote the observed feature  matrix for the squared reward, i.e., the $t$-th row of $\Phi_{\sigma,n}$ is $\phi_{\sigma}(s_t,a_t)^T$, $t=1,\dots,n$.
We also denote $Y_{\mu, n}$ as the observed reward vector, i.e., the $t$-th row of $Y_{\mu,n}$ is $r_t(s_t,a_t)$, and $Y_{\sigma, n}$ as the observed squared reward vector whose $t$-th entry is $r_t(s_t,a_t)^2$, $t=1,\dots,n$. Then, the least-squares estimators for $\theta_{\mu}$ and $\theta_{\sigma}$ take the form
\[\hat{\theta}_{\mu, n} = (\Phi^T_{\mu, n} \Phi_{\mu, n})^{-1} \Phi^T_{\mu, n} Y_{\mu, n}
~\mbox{ and }~
\hat{\theta}_{\sigma, n} = (\Phi^T_{\sigma, n} \Phi_{\sigma, n})^{-1} \Phi^T_{\sigma, n} Y_{\sigma, n} \mbox{ respectively.}\]
In this case, $\mu_R(s,a)$ and $\sigma^2_R(s,a)$ can be estimated via 
\begin{equation}\label{eq:muK}
\hat{\mu}_{R, n}^K(s,a) = \phi_{\mu}(s,a)^T \hat{\theta}_{\mu,n} 
\end{equation}
and
\[(\hat{\sigma}^{K}_{R, n}(s,a))^2 = \phi_{\sigma}(s,a)^T \hat{\theta}_{\sigma, n} -(\phi_{\mu}(s,a)^T \hat{\theta}_{\mu, n})^2\]
respectively.

Consider a Markov chain $\tilde X=(s,a)$ with transition kernel
\[\tilde P^{\pi}(s_2, a_2|s_1,a_1)=P(s_2|s_1,a_1)\pi(a_2|s_2).\]
We denote $\tilde \S=\{x_1, \dots, x_{m_sm_a}\}$ as the state space of $\tilde X$, where the states are indexed in the same order as they appear in the rows of $\tilde P^{\pi}$. For any state $\tilde x\in\tilde \S$, denote its $(s,a)$ representation as $(s(\tilde x), a(\tilde x))$. Recall that $w$ is the stationary distribution of $\tilde X$.
Let 
\[H(\tilde X) =\left(r(s(\tilde X),a(\tilde X))-\mu_R(s(\tilde X),a(\tilde X))\right)\phi_{\mu}(s(\tilde X), a(\tilde X)).
\]

\begin{lemma} \label{thm:CLT_kernel_R}
Assume $\mathbb{E}_{w}[(H(\tilde X))^2] <\infty$. Then, for any $(s,a)\in \tilde\S$,
\[ \sqrt{n}(\hat{\mu}_{R, n}^K(s,a) - \mu_R(s,a)) \Rightarrow \mathcal{N}({\bf 0}, \phi_{\mu}(s,a)^T\Upsilon_{\mu}^{-1}\Sigma_{H}\Upsilon_{\mu}^{-1}\phi_{\mu}(s,a)), \]
where 
\[\Upsilon_{\mu} = \mathbb{E}_w[\phi_{\mu}(\tilde X_0)(\phi_{\mu}(\tilde X_0))^T]\] 
and    
\[
\Sigma_{H} = \Var_w (H(\tilde X_0)) + 2\sum^{\infty}_{t=1} Cov_w(H(\tilde X_0), H(\tilde X_t)).
\] 
Specifically, 
\[\begin{split}
\Var_{w} (H(\tilde X_0)) &= \sum_{(s,a)\in \tilde\S} w_{s,a} H(s,a)H(s,a)^T\\ 
&=  \sum_{(s,a)\in \tilde\S} w_{s,a} \sigma_R^2(s,a) \phi_{\mu} (s,a)\phi_{\mu}(s,a)^T \\
&= \sum_{(s,a)\in \tilde\S} w_{s,a} \left(\phi_{\sigma}(s,a)^T\theta_{\sigma}\right) \phi_{\mu} (s,a)\phi_{\mu}(s,a)^T\\
\end{split}\]
and
\[\sum^{\infty}_{t=1} \Cov_w(H(\tilde X_0), H(\tilde X_t))
= \sum_{x_k,x_l\in \tilde \S}  w_{s(x_k),a(x_k)} H(s(x_k),a(x_k))H(s(x_l),a(x_l)) \left\{(I-\tilde{P}^{\pi})^{-1}\right\}_{k,l}.\]
\end{lemma}



Note that the asymptotic variance $\phi_{\mu}(s,a)^T\Phi_{\mu}^{-1}\Sigma_{H}\Phi_{\mu}^{-1}\phi_{\mu}(s,a)$ can be estimated using plug-in estimators. Here the sampling policy will affect $\Sigma_H$ through the transition matrix $\tilde P^{\pi}$ and the stationary distribution $w$. 

Based on the above analysis, the kernel-based Q-value estimates can be expressed as the empirical fixed point of the Bellman operator
\[
\hat Q_n^K=\mathcal{T}_{\hat P_n^K,\hat\mu_{R,n}^K}(\hat Q_n^K),
\]
where $\hat P_n^K$ is defined in \eqref{eq:PK} and $\hat\mu_{R,n}^K$ is defined in \eqref{eq:muK}.
Combining Lemmas \ref{thm:CLT_kernel_P} and \ref{thm:CLT_kernel_R}, we have the following central limit theorem for $\hat Q^K_n$.


\begin{theorem} \label{thm:CLT_kernel_Q}
Under assumptions  \ref{assum:P_kernel} and  \ref{assum:R_kernel}, assume $\mathbb{E}_{\bar\omega}[G^2_v(X))] < \infty$ and $\mathbb{E}_{w}[H^2(\tilde X)]<\infty$. Then,
\[ \sqrt{n}(\hat{Q}^K_n - Q) \Rightarrow \mathcal{N}({\bf 0},\Sigma^K) \mbox{\ \ as\ \  $n\to\infty$},\] where
\begin{equation}
	\Sigma^K = (I-\gamma \tilde{P}^{\pi^{*}})^{-1}(\phi_{\mu}^T\Upsilon_{\mu}^{-1}\Sigma_{H}\Upsilon_{\mu}^{-1}\phi_{\mu} + \gamma^2V_D F_{\phi, \psi} \Upsilon_E^{-1}\Sigma_{G}\Upsilon_E^{-1}  F^T_{\phi, \psi} V_D) ((I-\gamma \tilde{P}^{\pi^{*}})^{-1})^T,\label{formula var}
	\end{equation}
$\phi_{\mu}$ is an $(m_sm_a)\times d_1$ matrix whose $i$-th row is $\phi_{\mu}(s_i,a_i)$, 	
$F_{\phi, \psi} $ is an $(m_s^2 m_a)\times k_1$ matrix whose the $i$-th row is $F_{\phi, \psi}(s_i, a_i, s_i^{\prime})$, 
\[V_D = \begin{pmatrix} (V^{*})^T  & & & & \\
	& \ddots & & &\\
	& & (V^{*})^T & &\\
	& & & \ddots & \\ & & & & (V^{*})^T
	\end{pmatrix},\]
$\Upsilon_E$ and $\Sigma_G$ are defined in Lemma \ref{thm:CLT_kernel_P}, 
$\Upsilon_{\mu}$ and $\Sigma_H$ are defined in Lemma \ref{thm:CLT_kernel_R}.
\end{theorem}

\subsection{Q-OCBA Variants}
In this section, we develop the Q-OCBA variants for the nearest-neighbors approximation and the kernel representation. 

We first consider the nearest-neighbors approximate value iteration. 
Based on Theorem \ref{thm:clt_large},
let 
\[h^M_{ij}=\left(Q^M(i,a^*(i))-Q^M(i,j)\right)^2/\sigma_{\Delta Q^M}^2(i,a^*(i),j),\]
where 
\[\sigma^2_{\Delta Q^M}(i,a^*(i),j) = (e_{(i-1)m_a+a^*(i)}-e_{(i-1)m_a+j})^T\Sigma^M_{\mathcal{S}_0} (e_{(i-1)m_a+a^*(i)}-e_{(i-1)m_a+j}).\]
Then, we can modify $Q$-OCBA to $Q^M$-OCBA, where our goal is to find a sampling policy that solves 
\begin{equation}\label{eq:opt_M_2}
\max_{w\in\mathcal{W}_{\eta}} \min_{i\in \mathcal{S}}\min_{j\in \mathcal{A},j\neq a^*(i)}h^M_{ij}.
\end{equation}

Next, we consider the kernel representation. Based on Theorem \ref{thm:CLT_kernel_Q},
let 
\[h^K_{ij}=\left(Q^K(i,a^*(i))-Q^K(i,j)\right)^2/\sigma_{\Delta Q^K}^2(i,a^*(i),j),\]
where 
\[\sigma^2_{\Delta Q^K}(i,a^*(i),j) = (e_{(i-1)m_a+a^*(i)}-e_{(i-1)m_a+j})^T\Sigma^K (e_{(i-1)m_a+a^*(i)}-e_{(i-1)m_a+j})\]
Then, we can modify Q-OCBA to $Q^K$-OCBA, where our goal is to find an exploration policy that solves 
\begin{equation}\label{eq:opt_K_2}
\max_{w\in\mathcal{W}_{\eta}} \min_{i\in \mathcal{S}}\min_{j\in \mathcal{A},j\neq a^*(i)}h^K_{ij}.
\end{equation} 

In both cases discussed above, we can modify sequential updating rule in Algorithm \ref{alg:qocba} to develop the corresponding algorithmic implementations. In particular, for $Q^M$-OCBA, we replace \eqref{eq:optq} in Algorithm \ref{alg:qocba} with \eqref{eq:opt_M_2}. For $Q^K$-OCBA, we replace \eqref{eq:optq} with \eqref{eq:opt_K_2}.

\section{Numerical Experiments}\label{sec:numerics}
In this section, we conduct numerical experiments to support our large-sample results in Section \ref{sec:ucq}, demonstrate the performance of Q-OCBA against some benchmark exploration policies, and analyze how to tune the parameters when implementing Q-OCBA. We use the RiverSwim problem  with $m_s$ states and two actions at each state: swimming left (0) and swimming right (1). Figure \ref{fig:swim} provides a pictorial illustration of the problem. 
The triplet above each arc in Figure \ref{fig:swim} represents i) the action, ii) the transition probability to the next state given the current state and action, and iii) the reward under the current state and action. 
For example, in $(1,0.3,0)$, $1$ is the action of swimming right, $0.3$ is the probability of moving to the next state as the directed arc indicates, and $0$ is the reward for the current state-action pair.
Note that, in this problem, rewards are  only given at the left and right boundary states, where the left-boundary reward, $r_L<10$, will be varied in our experiments. We set the discount factor $\gamma=0.95$.

Note that swimming
to the right (against the current of the river) will more often than not leave the agent in the same state, but will sometimes
move the agent to the right, and with a much smaller probability move the agent to the left. Swimming to the left (with the current)
always succeeds in moving the agent to the left, until the leftmost state is reached at which point swimming to the left
yields a small reward (i.e., $r_L<10$). The agent receives a much larger reward for reaching the rightmost state. This MDP requires a sequence of appropriate actions in order to explore effectively to the right, and is challenging to learn when nothing is known at the beginning. Thus, it is a classic example to  study exploration policies (see, e.g., \citealt{strehl2008analysis, osband2013more}).



\begin{figure}[ht]
	\begin{center}
		\caption{RiverSwim Problem} \label{fig:swim}
		\includegraphics[width=0.8\textwidth]{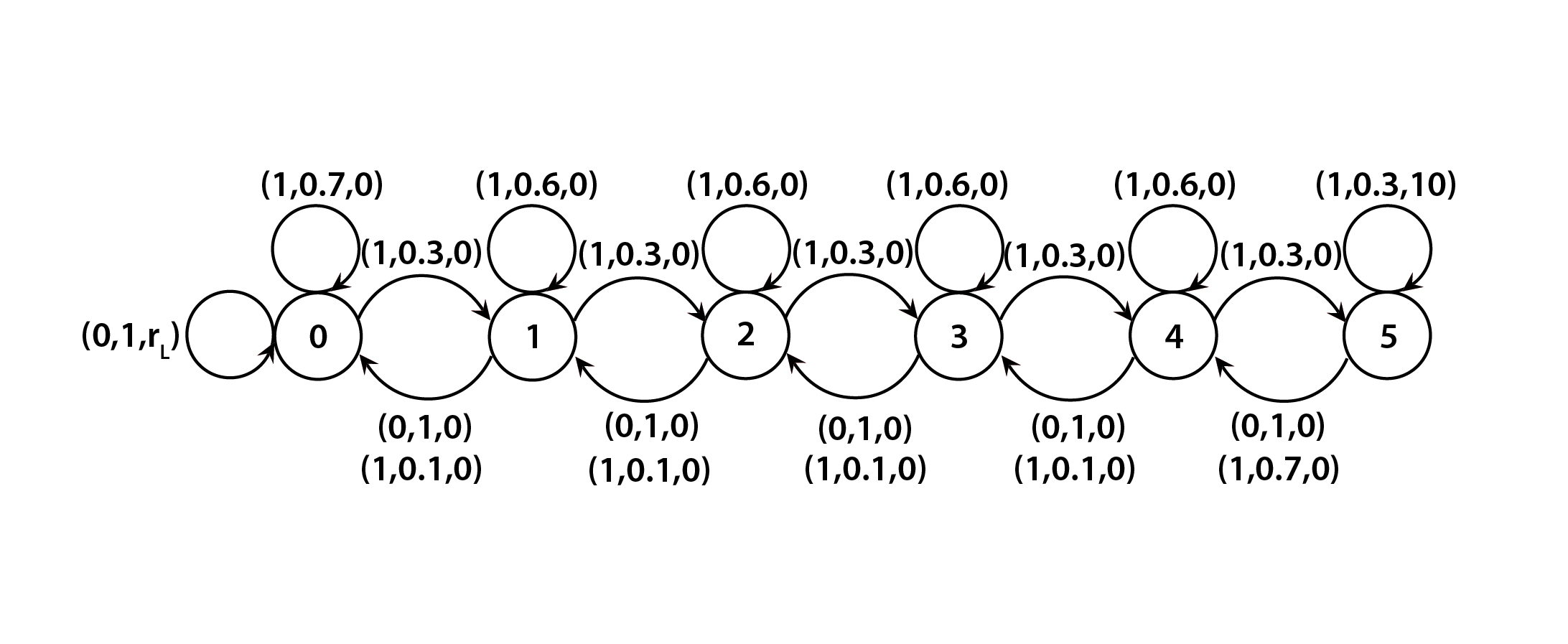}
	\end{center}
\end{figure}

\subsection{Statistical Quality of Interval Estimations}\label{sec:interval}
We first demonstrate the validity of our large-sample asymptotic results.
All coverage rates reported in our numerical experiments are estimated using $10^3$ independent repetitions of the procedure. 
The estimation errors (half width of the 95\% confidence intervals (CIs)) are around $0.01$.
Our target confidence level (coverage rate) is $95\%$.


We start with the RiverSwim problem with a small state space, i.e., $m_s=6$ as depicted in Figure \ref{fig:swim}, and $r_L=1$.
Table \ref{tab:cov_supp} reports the coverage rates of the asymptotically valid $95\%$ CI  
for $Q$, $V^*$, and $\chi^*$, constructed based on Theorem \ref{thm:clt_basic} and Corollary \ref{cor:V_R_basic}.
The data collection policy we employ is a random exploration policy under which we swim right with probability $0.8$ at each state, i.e., $\pi(1|s)=0.8$. 
We observe that the coverage rate approaches the nominal $95\%$ as the number of observations $n$ increases.
Specifically, when $n$ is $10^2$, the coverage rates are all around $30\%$, when $n$ is $5\times 10^2$, they are around $90\%$, and when $n\geq 10^3$, they are around $95\%$. These suggest a sample size of $10^3$ is enough to elicit our asymptotic results 
in this problem.
\begin{table}[H]
	\caption{Coverage for $Q(s,a)$, $V^{*}$ and $\chi^{*}$ values using exact tabular update} \label{tab:cov_supp}
	\centering
	\begin{tabular}{c|c|c|c|c}\hline
		$n$ & $10^2$ & $5\times 10^2$ & $ 10^3$ & $ 10^4$   \\ \hline
		$Q(1,0)$  & $0.2$ & $0.91$ &
		$0.95$ &
		$0.95$    \\ 
		$Q(3,1)$  & $0.31$ & $0.89$ & $0.95$  &
		$0.95$     \\ 
		$Q(6,0)$  & $0.30$ & $0.87$ & $0.94$  &
		$0.95$     \\ \hline
		$V^{*}(2)$  & $0.29$ & $0.89$ & $0.94$  &
		$0.96$      \\ 
		$V^{*}(4)$  & $0.30$ & $0.90$ & $0.94$  &
		$0.95$      \\ 
		$V^{*}(5)$  & $0.29$ & $0.88$ & $0.94$  &
		$0.95$      \\ \hline
		$\chi^{*}$  & 
		$0.30$ & $0.88$ &  $0.94$  &
		$0.95$      \\ \hline
		
	\end{tabular}
\end{table}

We next consider the RiverSwim problem with a large state space, i.e., $m_s=31$, and $r_L=1$.
We use the linear interpolation approximation with $\mathcal{S}_0=\{ 1,4,\dots,28,31\}$. 
Table \ref{tab:cov_linear} reports
the coverage rates of the asymptotically valid $95\%$ CI
for $Q$, $V^*$, and $\chi^*$, constructed based on Theorem \ref{thm:clt_large}, with different sample size and under different exploration policies.
In particular, we vary the values of $\pi(1|s)$ from $0.8$ to $0.9$. 

Compared to the exact update with a small state space, the coverage-rate convergence for the approximate update with a larger state space appears slower. Specifically, comparing Tables \ref{tab:cov_supp} and \ref{tab:cov_linear} that use the same random exploration with $\pi(1|s)=0.8$, 
we note that while the nominal coverage is obtained when $n=10^4$ in the exact update with a smaller state space for all studied quantities, this sample size is not enough for approximate update with $m_{s_0}=11$, where it appears that we need $n$ to be of order $10^7$ to obtain the nominal coverage. 

We also note that, when the coverage is very far from the nominal level, discrepancies can show up among the estimates of $Q$, $V^*$, and $\chi^*$. For example, when $\pi(1|s)=0.8$ and $n=10^4$, the coverage rates of $Q$ and $V^*$ are around $47\%-49\%$ but that of $\chi^*$ is as low as $2\%$, and when $\pi(1|s)=0.9$ and $n=10^6$, the coverage rates of $Q$ and $V^*$ are around $33\%-35\%$ but that of $\chi^*$ is only $4\%$ . However, when the coverage rate is close to $95\%$, all these quantities appear to attain this accuracy simultaneously in all the cases considered. Nonetheless, the convergence behaviors predicted by Theorem \ref{thm:clt_large} are observed to hold.

Furthermore, Table \ref{tab:cov_linear}  shows that the rates of convergence to the nominal coverage are quite different for different values of $\pi(1|s)$. The convergence rate when $\pi(1|s)=0.85$ seems to be the fastest, with the coverage rate close to $95\%$ already when $n=10^5$. On the other hand, when $\pi(1|s)=0.8$, the coverage rate is close to $95\%$ only when $n=10^6$, and when $\pi(1|s)=0.9$, even $n=10^7$ is not large enough to have the target coverage rate. These caution that estimation quality can be quite sensitive to the exploration policy (the quality of data collected). We investigate the efficiency of different exploration policies further in the next subsection.


\begin{table}[H]
	\caption{Linear interpolation in approximate value iteration} \label{tab:cov_linear}
	\centering
	\begin{tabular}{c|c|c|c|c|c}\hline
		& $n$ & $10^4$ & $10^5$ & $ 10^6$ & $10^7$   \\ \hline
		\multirow{3}{*}{$\pi(1|s) = 0.8$} &
		Average $Q$ coverage & $0.47$ &  $0.73$  &
		$0.93$  & $0.95$    \\ 
		&Average $V^{*}$ coverage  & $0.49$ &  $0.74$  &
		$0.94$    & $0.95$  \\ 
		& $\chi^{*}$ coverage  & 
		$0.02$ &  $0.37$  &
		$0.93$   & $0.94$   \\ \hline
		\multirow{3}{*}{$\pi(1|s) = 0.85$} & Average $Q$ coverage& $0.45$ &  $0.92$  &
		$0.94$  & $0.95$   \\ 
		&Average $V^{*}$ coverage & $0.48$ &  $0.93$  &
		$0.94$   & $0.95$   \\ 
		&$\chi^{*}$ coverage & 
		$0.12$ &  $0.91$  &
		$0.95$   & $0.94$   \\ \hline
		\multirow{3}{*}{$\pi(1|s) = 0.9$}  & Average $Q$ coverage  & $0.30$ &  $0.37$  &
		$0.33$   & $0.69$   \\ 
		&Average $V^{*}$ coverage   & $0.34$ &  $0.39$  &
		$0.35$   & $0.70 $   \\
		&$\chi^{*}$ coverage  & 
		$0.01$ &  $0.01$  &
		$0.04$  & $0.55$    \\ \hline
	\end{tabular}
\end{table}

\subsection{Efficiency of Exploration Policies} \label{sec:policy_num}

In this section, we investigate the efficiency of Q-OCBA defined in Algorithm \ref{alg:qocba}.
We first compare Q-OCBA to other benchmark policies for the task of learning the optimal policy. In particular, given a sampling budget, we compare which policy would be able to collect a more ``informative" data set to learn the optimal policy. Note that Q-OCBA and other benchmark policies are used to collect data, and the optimal policy is trained offline using the data collected under these policies (i.e., by solving the empirical fixed point of the Bellman operator). We then discuss how to tune Q-OCBA, e.g., choosing the number of iterations $K$.

\subsubsection{Comparison with other benchmark policies} \label{sec:num_compare}
We compare Q-OCBA to four benchmark policies: i) $\epsilon$-greedy with $\epsilon =0.2$, ii) random exploration (RE) with $\pi(1|s)=0.6$ and  $\pi(1|s)=0.8$, 
iii) UCRL2 (a variant of UCRL) with $\delta = 0.05$ \citep{jaksch2010near}, and iv) PSRL \citep{osband2013more}.  
We do sequential updating for Q-OCBA, PSRL, and  $\epsilon$-greedy. 

We first set $m_s=6$, and vary the values of $r_L$. $r_L$ is either set as a constant or a Gaussian random variable. $\hat{P}_{B_0}(s'|s, a)$ is initialized to be  $1/m_s$, and $\hat{\mu}_{R, B_0}(s,a)$ and $\hat{\sigma}^2_{R, B_0}(s,a)$ are initialized to be $1$ for all $(s, a)$ pairs.
Figure \ref{fig:seqm} compares the probability of obtaining the optimal policy (probability of correct selection) using data collected under different policies. The probability of correct selection is estimated using $10^3$ independent repetitions of the procedures and the estimation errors (half width of the 95\% confidence intervals) are around 0.01.
We observe that Q-OCBA outperforms the other policies in almost all cases. For example, when $r_L=1$ and $n=2000$, the probability of correct selection using Q-OCBA is $4$ times higher than that using PSRL, and 25\% higher than that using $\epsilon$-greedy. 
By carefully tuning the RE parameter $\pi(1|s)$, we find that $\pi(1|s)=0.8$ achieves comparable performance as Q-OCBA. However, when $\pi(1|s)$ is not optimally chosen, e.g., $\pi(1|s)=0.6$, the performance of RE is significantly worse than Q-OCBA. For example, when $r_L=1$ and $n=2\times 10^3$, RE(0.8) leads to a $95\%$ probability of correct selection, while RE(0.6) only gives a $72\%$ probability.
Note that there is no systematic way to tune the parameter of RE other than trial and error. In this sense Q-OCBA is more robust than RE.

We also observe that $\epsilon$-greedy and PSRL perform much worse when $r_L=3$ compared to when $r_L=1$.
One possible explanation is that for  $r_L=1$, the $(s,a)$ pairs that need to be explored more also tend to have larger Q-values. However, as $r_L$ increases,
the corresponding changes in the Q-values would change the exploration ``preference" of $\epsilon$-greedy and PSRL. 
On the other hand, as the underlying stochasticity of the system does
not change with $r_L$, the states that need more exploration remain unchanged.
Thus, there is a misalignment between the Q-values and the $(s,a)$ pairs that need more exploration. 
In contrast, the performance of Q-OCBA is very stable against different values of $r_L$.
The superiority of Q-OCBA in these experiments comes as no surprise to us. The benchmark policies like UCRL and PSRL are designed to minimize regret which involves balancing the exploration-exploitation trade-off. On the other hand, Q-OCBA focuses on efficient exploration only, i.e., our goal is to minimize the probability of incorrect policy selection. This is achieved by
carefully utilizing the variance information gathered from the previous stages, and is made possible by our derived asymptotic variance formulas. 

Lastly, we observe that randomness in reward contaminates the estimation. For example, when $r_L=1$ (Figure 2(a)), 
Q-OCBA is able to achieve a higher than $90\%$ probability of correct selection when the sample size is $10^3$. However, when $r_L\sim\mathcal N(1,10^2)$(Figure 2(c)), even with a sample size of $7\times 10^3$, the probability of correct selection under Q-OCBA is only around $80\%$. Nevertheless, even with a high variability in reward, Q-OCBA still outperforms other benchmark policies. For example, when $r_L\sim \mathcal{N}(1,10^2)$ and $n=7\times 10^3$, Q-OCBA achieves a 14\% higher probability of correct selection than the second best, which in this case is RE(0.6). 

\begin{figure}[htb]
	\caption{Probability of correct selection (PCS) with different sampling budgets and different values of $r_L$.} 
	\label{fig:seqm}\centering
	\subfloat[$r_L=1$]{
		\includegraphics[width=0.45\textwidth]{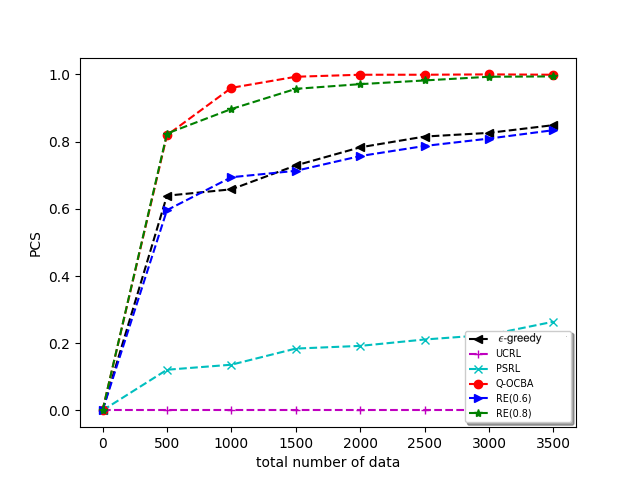}
	} 
	\subfloat[$r_L=3$]{
		\includegraphics[width=0.45\textwidth]{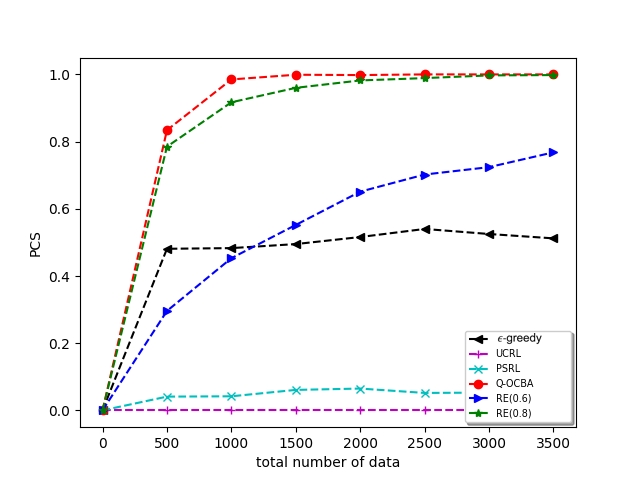}
	} \\
	\subfloat[$r_L\sim \mathcal{N}(1,10^2)$]{
		\includegraphics[width=0.45\textwidth]{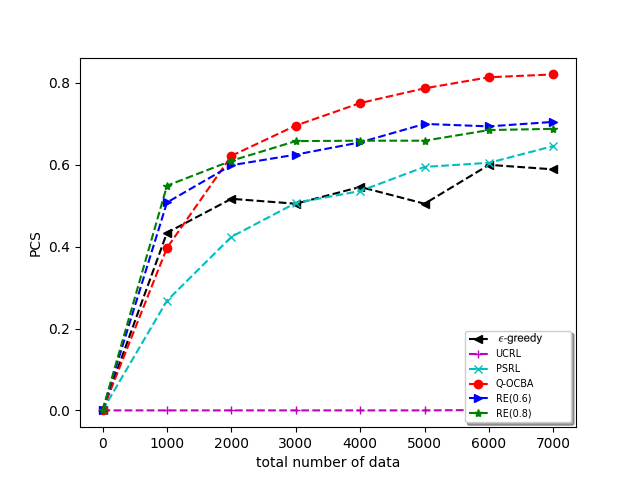}
	} 
	\subfloat[$r_L\sim \mathcal{N}(3, 10^2)$]{
		\includegraphics[width=0.45\textwidth]{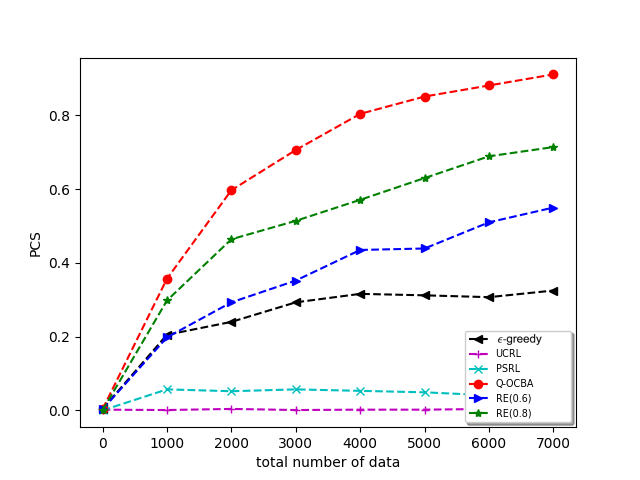}
	}
\end{figure}

In practice, when we have a two-stage implementation where we collect experience in the first stage and deploy the trained policy to collect reward in the second stage, we would care more about the regret from the trained policy than the probability of obtaining the optimal policy itself. 
Let $\hat{\pi}^{*}_n$ denote the estimated optimal policy trained offline using a sample of size $n$ collected via a proper exploration policy, e.g., Q-OCBA. In this case, we define the future regret as
\[R=\rho^T(V^*- \E[V^{\hat{\pi}^{*}_n})].\]
In particular, $R$ measures the average loss from deploying the estimated optimal policy in the second stage. 

In Figure \ref{fig:seqfrm}, we compare the future regret for $\hat{\pi}^{*}_n$'s trained based on data collected under different 
first-stage exploration policies using the examples as those in Figure \ref{fig:seqm}.
We observe that Q-OCBA still outperforms most benchmark policies.
For example, when $r_L=1$ and $n=2\times 10^3$, the regret of the policy trained using data collected under Q-OCBA is only $1/160$ of that using PSRL, and $1/30$ of that using $\epsilon$-greedy.
Even though Q-OCBA is not designed to minimize $R$, it is intuitive that the probability of selecting the optimal policy is positively corrected with
$V^{\hat{\pi}^{*}_n}$, which matches our empirical observations.

\begin{figure}[htb]
	\caption{Future regret of Q-OCBA vs benchmarks as $n$ increases} 
	\label{fig:seqfrm}\centering
	\subfloat[$r_L=1$]{
		\includegraphics[width=0.5\textwidth]{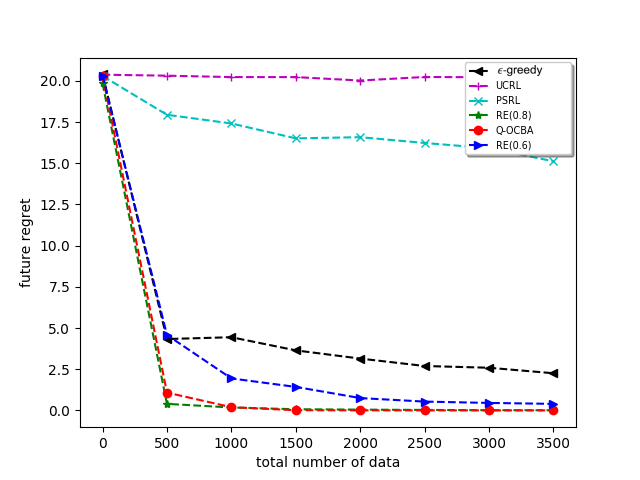}
	} 
	\subfloat[$r_L=3$]{
		\includegraphics[width=0.5\textwidth]{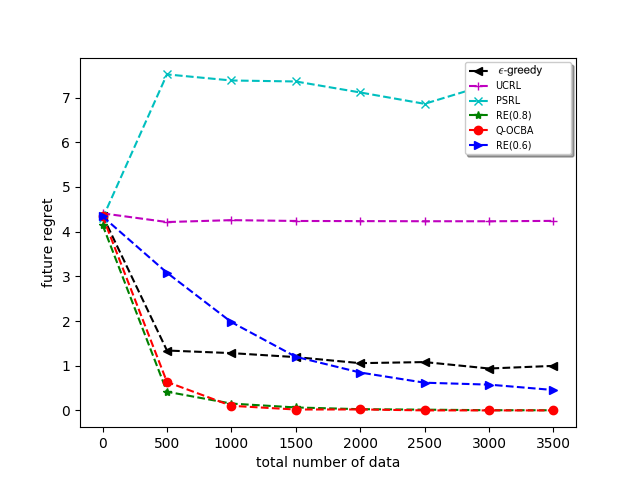}
	}\\
	\subfloat[$r_L\sim \mathcal{N}(1, 10^2)$]{
		\includegraphics[width=0.5\textwidth]{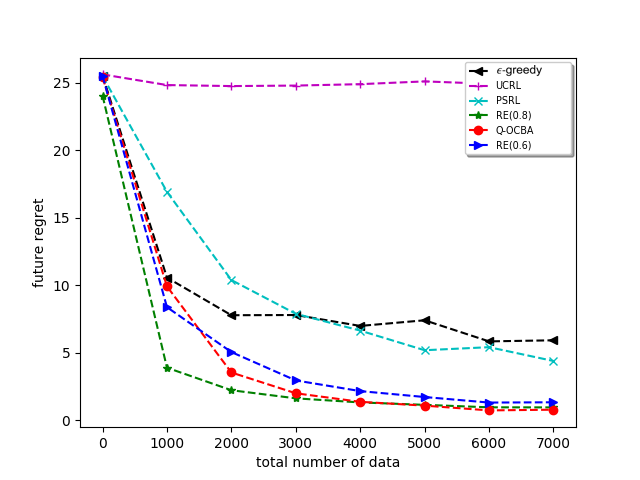}
	} 
	\subfloat[$r_L\sim \mathcal{N}(3, 10^2)$]{
		\includegraphics[width=0.5\textwidth]{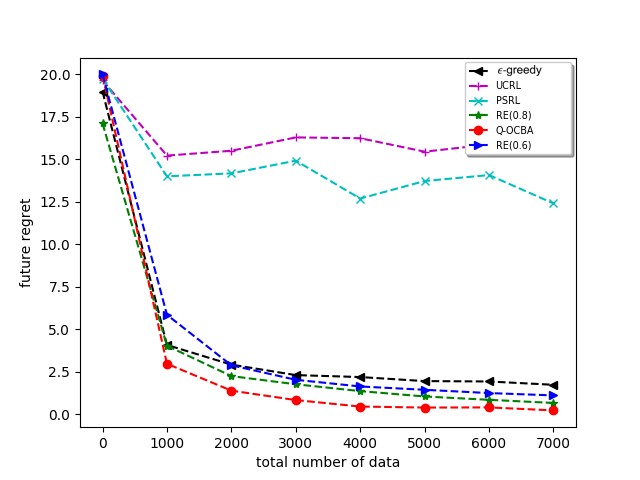}
	}
\end{figure}

Next, we vary the value of $m_s$ to test the performance of the algorithm for large state spaces. In particular, we fix $n=10^5$ and assume $r_L\sim \mathcal{N}(2, 1)$. For the sequential update (for Q-OCBA, PSRL, and $\epsilon$-greedy), we set $K=10$. For the initialization, we sample $\hat{P}_{B_0}$, $\hat{\mu}_{R, B_0}$ and $\hat{\sigma}^2_{R, B_0}$ from Uniform$[0,1]$ ($\hat{P}_{B_0}$ needs further normalization to be a valid transition matrix).

Figure \ref{fig:seqm_ls} compares the performance (probability of correct selection and future regret) under different data collecting polices for different values of $m_s$. We also calculate the performance of Q-OCBA with known system parameters (i.e., $h_{ij}$'s defined in \eqref{eq:hij} can be evaluated exactly). This oracle policy is referred to as Q-OCBA$_\text{known}$.

We observe from Figure \ref{fig:seqm_ls} that with a fixed sampling budget, the performance of all policies deteriorates as $m_s$ increases. This is because the problem is more difficult to learn as the state space grows larger. Q-OCBA$_\text{known}$ performs the best among all policies tested. Even when $m_s=55$, it still achieves a $52\%$ probability of correct selection. Q-OCBA consistently outperforms the other benchmark policies. For example, when $m_s=45$, Q-OCBA (with estimated parameters) achieves a $25\%$ probability of correct selection, while the best benchmark policy, which in this case is RE(0.8), only has an $11\%$ probability of correct selection. We also observe that the performance of UCRL, PSRL, and $\epsilon$-greedy deteriorate very quickly as $m_s$ increases. When $m_s\geq 15$, their probability of correct selection is almost zero. This suggests that with a limited sampling budget, when the state space is large, it is very important to optimize the exploration policy to collect  informative data for policy training. In this regard, Q-OCBA shows the potential to outperform in this example.



\begin{figure}[htb]
	\caption{Probability of correct selection (PCS) and future regret with different $m_s$ } 
	\label{fig:seqm_ls}\centering
	\subfloat[]{
		\includegraphics[width=0.5\textwidth]{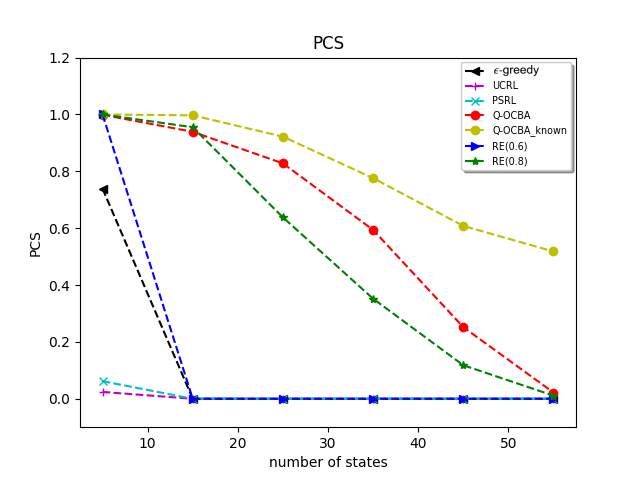}
	} 
	\subfloat[]{
		\includegraphics[width=0.5\textwidth]{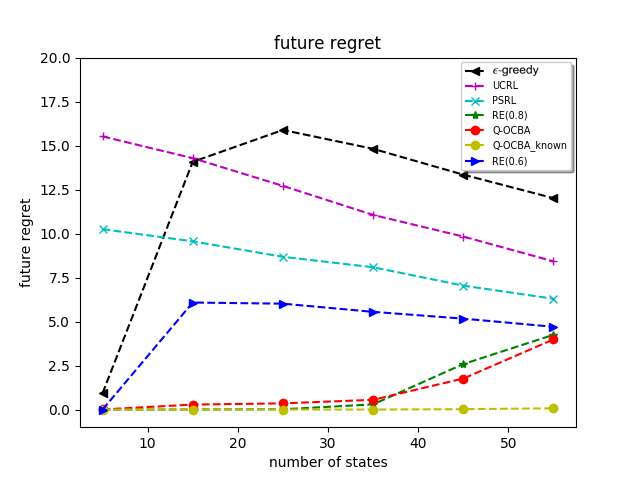}
	} \\
\end{figure}

\subsubsection{Tuning Q-OCBA} \label{sec:sen}

Q-OCBA requires very minimum tuning. Given a sampling budget, we need to specify the number of iterations $K$ and the budget for each iteration.
In addition, we need to provide some reasonable initializations for $\hat P(\cdot|s,a)$ and $\hat \mu_R(s,a)$. 

We note from our experiments that the performance of Q-OCBA is quite robust to different choices of initializations. 
This robustness stems from the updating criterion \eqref{eq:optq}. In particular, \eqref{eq:optq} leads to stationary distributions that are reasonably far from zero for all the state-action pairs. This stationary measure induces efficient exploration even when the initializations deviate substantially from the true values. In Table \ref{tab:PCS_comparison_seq_QOCBA}, we report the probability of correct selection under different initialization for RiverSwim with $m_s=6$ for different values of $r_L$. We set the initialization for $\hat P(\cdot|s,a)$ as uniform, 
i.e., for $1\leq s\leq 4$, $p^0(s^{\prime}|s,a)=1/3$ for $s^{\prime}\in\{s, s+1, s-1\}$, 
for $s=0$, $p^0(0|0,a)=1/2$ and $\hat P(1|0,a)=1/2$, and for $s=5$, $p^0(5|5,a)=1/2$ and $p^0(4|5,a)=1/2$.
We vary the initializations for $\hat \mu_R(s,a)$, for all $(s,a)\in\mathcal{S}\times\mathcal{A}$, from $0$ to $100$.
The sampling budget is $n=10^3$, the number of iterations is $K=10$, and the sampling budget is equally distributed among the $10$ iterations.
We observe from Table \ref{tab:PCS_comparison_seq_QOCBA} that across all problem instances and initializations, Q-OCBA achieves a very high probability of correct selection. Specifically, even in the extreme case where we set $u_R^0(s,a)=100$, when the true $r_L=1$, the probability of correct selection is $0.97$. We also note that the probability of correct selection is not monotone in the value of $r_L$, with $r_L=2$ leading to the lowest probability of correct selection.

\begin{table}[tbh]
	\caption{Probability of correct selection using Q-OCBA under different initializations, $n = 10^3$ and $K=10$} \label{tab:PCS_comparison_seq_QOCBA}
	\centering{
		\begin{tabular}{c|c|c|c|c}\hline
			$r_L$	 &  
			$u_R^0(s,a)=0$ & $u_R^0(s,a)=1$
			& $u_R^0=10$ & $u_R^0(s,a)=100$ 
			\\ \hline
			$1$ &
			$0.98$ & $0.95$ & $0.96$ & $0.97$ 
			\\ 
			
			$2$ 
			& $0.91
			$ &  $0.91$ 
			& $0.87$ 
			& $0.84$ 
			 \\ 
			
			$3$ 
			& $0.98$
			&  $0.97$ 
			& $0.99$ 
			& $0.98$ 
			\\ \hline
			
		\end{tabular}}
	\end{table}

We next study the choice of $K$ -- the number of iterations, which determines how often the exploration policy is updated during data collection.
Figure \ref{fig:num_stage_QOCBA} compares the probability of correct selection for different values of $K$ based on RiverSwim with $m_s=6$.
To focus on the effect of $K$,  the sampling budget is equally distributed among the $K$ iterations. We observe that as $K$ increases, the probability of correct selection under Q-OCBA increases. Specifically, with $10^4$ samples, when $K=2$, the probability of correct selection is only around $18\% - 25\%$; when $K=10$, the probability of correct selection increases to above $85\%$ in all cases, with the probability almost equal to 1 for $r_L=1$ or $3$.
However, there is a diminishing benefit of increasing $K$. For example, when increasing $K$ from 2 to 4, the probability of correct selection increases by more than $0.5$. In contrast, when increasing $K$ from 6 to 8, the probability only increases by only $0.03$.
We also note that when the sampling budget $n$ is small or when there is a lot of uncertainty in $r_L$, we benefit more from having a larger value of $K$. In particular, when comparing Figure \ref{fig:num_stage_QOCBA}(a) to \ref{fig:num_stage_QOCBA}(c), $K=4$ leads to a probability of correct selection of $0.82$ when $n=10^3$ versus $0.98$ when $n=10^4$. When $n=10^3$, to achieve a higher than $0.9$ probability of correct selection, we need $K\geq 6$.
As a general rule of thumb, we suggest setting $K$ between $6$ and $10$, as updating the policy at each iteration, i.e., solving \eqref{eq:optq}, incurs some computational cost.  

\begin{figure}[htb]
	\caption{Number of stages against probability of correct selection (PCS)} 
	\label{fig:num_stage_QOCBA}\centering
	\subfloat[$r_L=1, n=10^3$]{
		\includegraphics[width=0.45\textwidth]{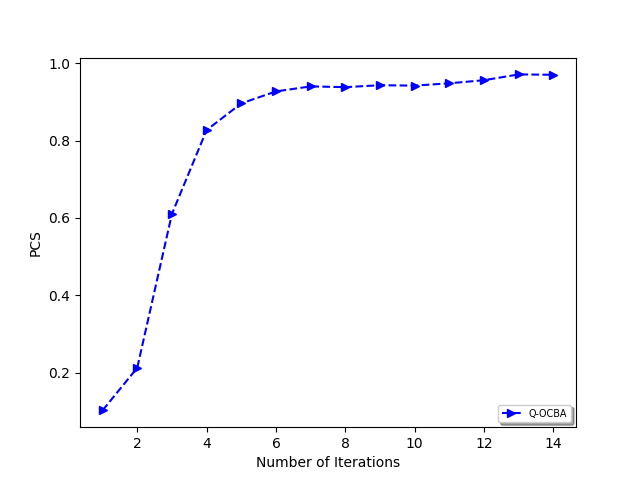}
	} 
	\subfloat[$r_L=3, n=10^3$]{
		\includegraphics[width=0.45\textwidth]{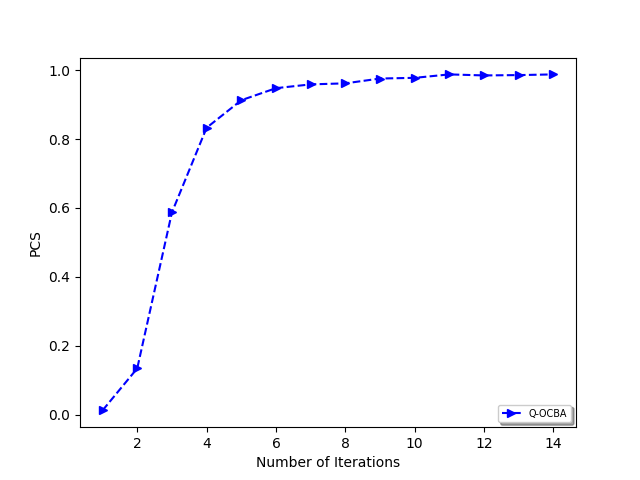}
	}\\
	
	\subfloat[$r_L=1, n=10^4$]{
		\includegraphics[width=0.45\textwidth]{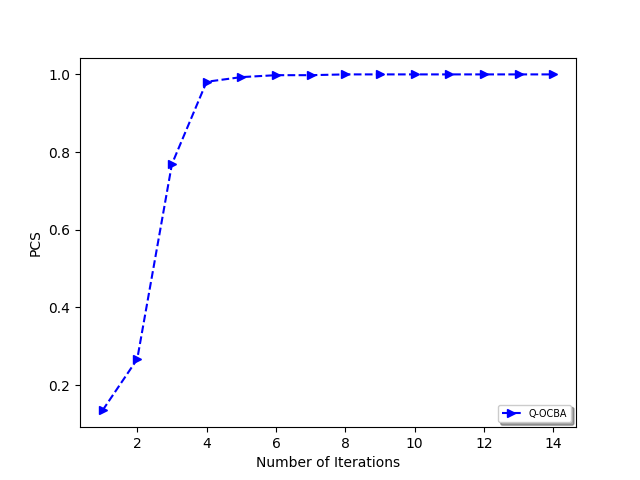}
	} 
	\subfloat[$r_L=3, n=10^4$]{
		\includegraphics[width=0.45\textwidth]{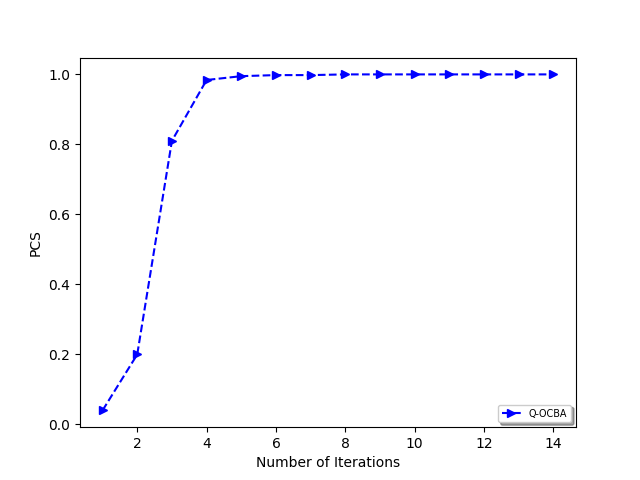}
	}\\
	\subfloat[$r_L\sim N(1,10^2), n=10^4$]{
		\includegraphics[width=0.45\textwidth]{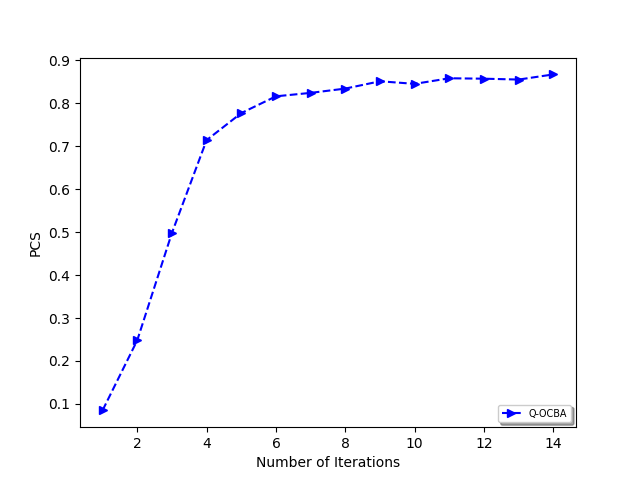}
	} 
	\subfloat[$r_L\sim N(3,10^2), n=10^4$]{
		\includegraphics[width=0.45\textwidth]{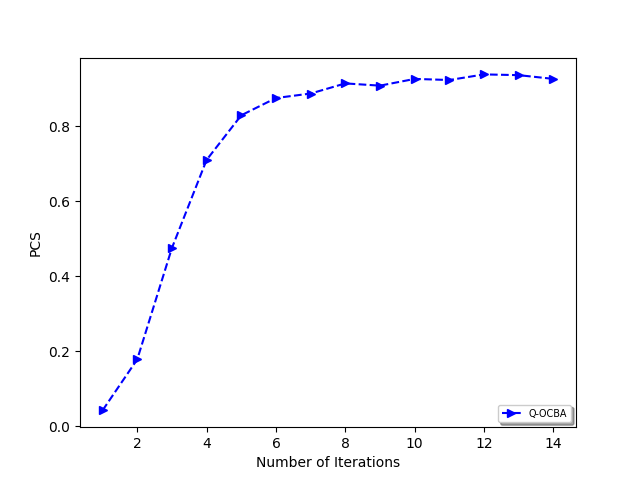}
	}~~~ \\

\end{figure}

\section{Concluding Remarks} \label{sec:con}
In this paper, we investigate the large-sample behaviors in the
estimation of Q-values and related quantities in RL. We establish the corresponding
central limit theorem with an explicit  characterization of the asymptotic variance. Utilizing the asymptotic variance obtained, we further develop a pure exploration policy, Q-OCBA, that maximizes the worst-case relative discrepancy among the estimated Q-values relative to the optimal action. In our numerical experiments, we compare Q-OCBA with various benchmark policies in terms of the probability of correctly selecting the optimal policy and the future regret. We observe that Q-OCBA outperforms benchmark methods according to these pure exploration criteria.
Our work is among the few that investigate statistical uncertainty quantification for RL and develop policies for pure exploration purposes. 

There are several limitations of our current developments that would lead to interesting future research directions.
First, compared to existing finite-sample bounds for RL estimates, our asymptotics result in tighter bounds and explicit characterization of the constants (i.e., the asymptotic variances). However, uncertainty quantification based on asymptotics also has its limitations. With a finite sample size, we may not be able to achieve the desired confidence level and it often requires trial and error to decide how many samples are required to achieve close-to-asymptotic performance.

Second, our exploration policy requires estimating the transition matrix and the reward distribution for each state-action pair. This may need a large amount of data and can be computationally expensive when the state space and/or action space is large. We consider two approximation schemes: nearest neighbors and kernel representation. However, in many large-scale RL problems, it is highly non-trivial to find a good approximation or lower-dimensional representation of the underlying MDP. How to find good approximations for large-scale problems and develop efficient exploration policies accordingly require substantial future developments. We believe our framework provides a promising starting point.

Lastly, the relative discrepancy maximization criteria in Q-OCBA has its own limitations. The mean-variance trade-off in OCBA could lead to misleading results in certain low-confidence scenarios (see, for example, \cite{shin2016tractable,peng2017gradient}). Other criteria for information gaining such as knowledge gradient \citep{ryzhov2012knowledge} or posterior probability of correct selection \citep{peng2018ranking} may be employed to achieve better performance. However, how to adapt these criteria to the RL setting requires further investigation.
In addition, there is in general a lack of finite-sample performance bounds for OCBA algorithms. Developing meaningful performance guarantees when dealing with a relatively small sample size would also be an interesting research direction. This can provide further guidance on how to fine-tune the algorithm parameters.

\begin{APPENDICES}

\section{Proofs of the results in Section \ref{sec:stand}}

In all the subsequent proofs, we shall treat $P$ and $\hat P_n$ as $N m_s$-dimensional vectors following the index rule:
$P((i-1)N+(j-1)m_s+k) = P(k|i,j)$, $\hat P_n((i-1)N+(j-1)m_s+k) = P_n(k|i,j)$

\subsection{Proof of Theorem \ref{thm:clt_basic}}
The proof of Theorem \ref{thm:clt_basic} is based on the delta method and the implicit function theorem.

We first note that under Assumptions \ref{assum:ind} and \ref{assum:postive_stat}, we have
	\begin{equation}\label{eq:r_p_lln}
	[\hat{\mu}_{R,n}, \hat{P}_n] \rightarrow [\mu_R,P]  \mbox{ a.s.}
	\end{equation}
and
\[\frac{1}{n}\sum_{1\leq t\leq n} \mathds{1}(s_t=i, a_t=j) \rightarrow w((i-1)m_a + j) \mbox{ a.s. as $n\rightarrow\infty$}.\]
By Slutsky's theorem and Proposition 3 in \citep{trevezas2009variance}, 

	\begin{equation}\label{eq:r_p_clt1} 
	\sqrt{n}([\hat{\mu}_{R,n}, \hat{P}_n] - [\mu_{R},P]) \Rightarrow \mathcal{N}(0, \Sigma_{R,P}),
	\end{equation} 
	where $\Sigma_{R,P}$ is defined in \eqref{eq:Dp}
	and is an $(N+Nm_s)\times (N+Nm_s)$ matrix. 

	Second, define $F(Q',r',P')$ as a mapping from $\mathbb{R}^{N}\times \mathbb{R}^{N}\times \mathbb{R}^{N m_s}$ to $\mathbb{R}^{N}$, representing the fixed point equation for the Bellman operator \eqref{eq:bellman}:
	\begin{eqnarray*}
		F(Q',r',P')((i-1)m_a+j) 
		&=& Q'((i-1)m_a+j) -r'((i-1)m_a+j)\\
		&&-\gamma \sum_{1\leq k\leq m_s}P'((i-1)N +(j-1)m_s+k) ~ g_k(Q')
	\end{eqnarray*}  for $1\leq i \leq m_s$ and $1\leq j \leq m_a$, where $g_k(Q') = \max_{1\leq l\leq m_a} Q'((k-1)m_a+l)$, for $1\leq k \leq m_s$.
	
	Under Assumption \ref{assum:unique},
	there exists an open neighborhood of $Q$, which we denote as $\Omega$, 
	such that for any $Q' \in \Omega$, $\arg\max_{1\leq l\leq m_a} Q'((k-1)m_a +l)$ is still unique for each $1 \leq k \leq m_s$. Then, $g_k(Q')$ has all its partial derivatives exist and continuous. This implies that $F(Q',r',P')$ is continuously differentiable in $\Omega \times  \mathbb{R}^{N}\times \mathbb{R}^{N m_s} $.
	
	Denote the partial derivatives of $F$ as 
	\[\frac{\partial F}{\partial (Q',r' , P')} = \left[ \frac{\partial F}{\partial Q'} , \frac{\partial F}{\partial r'} , \frac{\partial F}{\partial  P'} \right]. \] 
	Note that $\frac{\partial F}{\partial Q'}$ is an $N\times N$ matrix. Denote its element at the $((i-1)m_a+j)$-th row, $((k-1)m_a+l)$-th column by 
	\begin{eqnarray*}
		\frac{\partial F_{(i-1)m_a+j}}{\partial Q'_{(k-1)m_a+l}} 
		&=& 
		\mathds{1}\left(i=k,j=l \right)-\gamma P'((i-1)N+(j-1)m_s + k)\\
		&& \times \mathds{1}\left(Q'((k-1)m_a+l) = \max_{1\leq u\leq m_a} Q'((k-1)m_a+u)\right).
	\end{eqnarray*} 
	Define $\tilde{P}'$ is an $N\times N$ matrix with 
	\begin{eqnarray*}
	&&\tilde{P}'((i-1)m_a+j,(k-1)m_a+l)\\ 
	&=&  P'((i-1)N+(j-1)m_s + k)\mathds{1}\left(Q'((k-1)m_a+l)= \max_{1\leq u\leq m_a} Q'((k-1)m_a+u)\right)
	\end{eqnarray*}
	for $1\leq i\leq m_s$, $1\leq j\leq m_a$,$1\leq k\leq m_s$, and $1\leq l\leq m_a$.
	Then we have
	\[\frac{\partial F}{\partial Q'} = I - \gamma \tilde{P}'.\] 
	
	Since all rows of $\tilde{P}'$ sum up to one,
	$\tilde{P}'$ can be interpreted as the transition matrix of a Markov chain with state space $\{(i, j): 1\leq i \leq m_s, 1\leq j \leq m_a \}$. 
	Then, $\frac{\partial F}{\partial Q'}$ is invertible for any $Q' \in \Omega$.
	%
	%
	This allows us to apply the implicit function theorem to the equation $F(Q,\mu_R, P) = 0$. 
	In particular, there exists an open set $U\subset\mathbb{R}^{N} \times \mathbb{R}^{N m_s}$, around $\mu_R\times P$, and a unique continuously differentiable function $\phi$: $U\rightarrow \mathbb{R}^{N}$, such that $\phi(\mu_R,P) = Q$, and for any $r'\times P'\in U$
	\[ F(\phi(r', P'), r', P') =0. \] 
	In addition, the partial derivatives of $\phi$  satisfy
	\[ \nabla \phi(\mu_R,P):=\left. \frac{\partial \phi}{\partial (r',P')}\right |_{{r'=\mu_R}, {P'=P}} = \left. -\left[\frac{\partial F}{\partial Q'}\right]^{-1} \left[ \frac{\partial F}{\partial r'}, \frac{\partial F}{\partial P'}\right] \right |_{{Q'=Q},{r'= \mu_R}, {P'=P}}.\]
For $\frac{\partial F}{\partial r'}$, we have
	\[\left. \frac{\partial F}{\partial r'} \right|_{{Q'=Q},{r'=\mu_R}, {P'=P}} = -I_{N \times N}.\]
%
	For $\frac{\partial F}{\partial P'}$, because
	\begin{eqnarray*}
		\left.\frac{\partial F_{(i-1)m_a+j}}{\partial P'_{(k-1)N+(l-1)m_s+v}} \right|_{{Q'=Q},{r'=\mu_R}, {P'=P}}
		&=& 
		-\gamma  \max_{0\leq u\leq m_a} Q((v-1)m_a+u) \mathds{1}\left(k=i, j=l \right) \\
		&=& -\gamma V^{*}(v) \mathds{1}\left(k=i, j=l \right), 
	\end{eqnarray*}
	for $1\leq i\leq m_s$, $1\leq j\leq m_a$,$1\leq k\leq m_s$,$1\leq l\leq m_a$, and $1\leq v\leq m_s$, 
	we can write
	\[\left. \frac{\partial F}{\partial P'}\right|_{{Q'=Q},{r'=\mu_R}, {P'=P}} = -\gamma \begin{pmatrix} (V^{*})^T  & & & & \\
	& \ddots & & &\\
	& & (V^{*})^T & &\\
	& & & \ddots & \\ & & & & (V^{*})^T
	\end{pmatrix},\]
	which is an $N\times Nm_s$ matrix.
	
	Now, from \eqref{eq:r_p_lln},
	by continuous mapping theorem, we have
	\[\phi(\hat{\mu}_{R,n}, \hat{P}_n) - \phi(\mu_R,P)\rightarrow 0 \mbox{ a.s. as $n\rightarrow\infty$},\] 
	which implies that $\hat{Q}_n \rightarrow Q$ a.s..
	
	From \eqref{eq:r_p_clt1}, using the delta method, we have
	\[	\sqrt{n}(\hat{Q}_n- Q)
		= \sqrt{n}(\phi(\hat{\mu}_{R,n}, \hat{P}_n) - \phi(\mu_R,P))
		\Rightarrow \mathcal{N}(0, \nabla \phi(\mu_R,P) \Sigma_{R,P} \nabla  (\phi(\mu_R,P))^T)\mbox{ as $n\to\infty$}.\]
	Plugging in the formula for $\nabla \phi(\mu_R,P)$ and $\Sigma_{R,P}$ derived above, we have
	\[\nabla \phi(\mu_R,P) \Sigma_{R,P} \nabla  \phi(\mu_R,P)^T=(I-\gamma \tilde{P}^{\pi^{*}})^{-1}W^{-1}[D_R + \gamma^2D_Q] ((I-\gamma \tilde{P}^{\pi^{*}})^{-1})^T.\]
\Halmos

\subsection{Proof of Corollary \ref{cor:V_R_basic}}
	Define $g_V(Q)$: $\mathbb{R}^{N} \to \mathbb{R}^{m_s}$ as 
	\[g_V(Q) = (g_1(Q), \dots, g_{m_s}(Q))=(V^{\pi^{*}}(1), \dots, V^{\pi^{*}}(m_s)),\] 
	which is continuously differentiable in an open neighborhood of $Q$. 
	Then we can apply the delta method to get
	\[\sqrt{n} (\hat{V}^{*}_n - V^{\pi^{*}})= \sqrt{n}(g_V(\hat{Q}_n) - g_V(Q)) \Rightarrow \mathcal{N}(0,\nabla g_V(Q) \Sigma (\nabla g_V(Q))^T) \mbox{ as $n\to\infty$}. \]
	Note that $\nabla g_V(Q)$ is a $m_s \times N$ matrix with $ \nabla g_V(Q)(i, (j-1)m_a+k)=  \mathds{1} \left(i=j,k=a^{*}(i)\right)$.
	
	To see where the explicit expression for $\Sigma_V$ is from, we will first rearrange the indexes such that we can write 
	\[\tilde{P}^{\pi^{*}} = \begin{pmatrix} P^{\pi^{*}} & {\bf 0} \\  $\#$ & {\bf 0} \end{pmatrix},\] 
	where we use $\#$ as a generic placeholder for quantitities that we do not need to characterize explicitly.
	Using this new indexing, we can write $\nabla g_V(Q) = [I, {\bf 0}]$, and
	\[
	D_R =\begin{pmatrix} D_R^{\pi^*} & {\bf 0} \\  {\bf 0} & \# \end{pmatrix}, 
	W=\begin{pmatrix} W^{\pi^{*}} & {\bf 0} \\ {\bf 0} & \# \end{pmatrix},
	D_Q =\begin{pmatrix} D_V^{\pi^{*}}& {\bf 0} \\ {\bf 0} & \# \end{pmatrix}.
	\]
	Then
	\begin{eqnarray*}
		\nabla g_V(Q) (I-\gamma \tilde{P}^{\pi^{*}})^{-1}
		&=& [I, {\bf 0}] \sum^{\infty}_{i=0} \gamma^i (\tilde{P}^{\pi^{*}})^i\\	
		&=& [I, {\bf 0}] \sum^{\infty}_{i=0} \gamma^i \begin{pmatrix} (P^{\pi^{*}})^i & {\bf 0} \\ \# & {\bf 0} \end{pmatrix} \\	
		&=& \sum^{\infty}_{i=0} \gamma^i \left[ (P^{\pi^{*}})^i, {\bf 0} \right]
		=	\left[(I-\gamma P^{\pi^{*}})^{-1}, {\bf 0} \right].	
	\end{eqnarray*}
	And thus,
	\begin{eqnarray*}
		&&\nabla g_V(Q) \Sigma (\nabla g_V(Q))^T\\
		&=&  \left[(I-\gamma P^{\pi^{*}})^{-1}, {\bf 0} \right] \begin{pmatrix} W^{\pi^{*}}& {\bf 0} \\ {\bf 0} & \# \end{pmatrix}^{-1}\left[\begin{pmatrix} D_R^{\pi^*} & {\bf 0} \\ {\bf 0} & \# \end{pmatrix} + \begin{pmatrix} D_V^{\pi^{*}}& {\bf 0} \\ {\bf 0} & \# \end{pmatrix}\right]	\left[(I-\gamma P^{\pi^{*}})^{-1}, {\bf 0} \right]^T \\
		&=& (I-\gamma P^{\pi^{*}})^{-1}(W^{\pi^{*}})^{-1}[D_R^{\pi^{*}} +  D_V^{\pi^{*}}]((I-\gamma P^{\pi^{*}})^{-1})^T.
	\end{eqnarray*}
	
	Lastly, the asymptotic normality of $\hat{\chi}^{\pi^{*}}_n$ follows from the delta method as well. \Halmos


\subsection{Proof of Corollary \ref{cor:V_basic_pi}}
The convergence results for $\hat{\mu}_{R,n}$ and $\hat{P}_n$ still hold in this case.
	Similar to the proof of Theorem \ref{thm:clt_basic}, we define $F^{\tilde{\pi}}$ as a mapping $\mathbb{R}^{m_s}\times \mathbb{R}^{N}\times \mathbb{R}^{Nm_s} \rightarrow \mathbb{R}^{m_s} $ to represent the corresponding fixed point equation, i.e.,
	\[F^{\tilde{\pi}} (V', r', P')(s) = V'(s) - \sum_{1\leq a\leq m_a} r'(s,a) \tilde{\pi}(a|s) - \gamma \sum_{1\leq a\leq m_a} \tilde{\pi}(a|s) \sum_{1\leq s' \leq m_s} P'(s'|s,a) V'(s').\] 
	Note that
	$F^{\tilde{\pi}} (V^{\tilde{\pi}}, \mu_R, P) = 0$, $F^{\tilde{\pi}}$ is continuously differentiable  and $I-\gamma P^{\tilde{\pi}}$ is invertible. We can thus apply the implicit function theorem. In particular, there exists an open set $U^{\tilde{\pi}}$ around $\mu_R\times P\in \mathbb{R}^{N} \times \mathbb{R}^{N m_s}$, and a unique continuously differentiable function $\phi^{\tilde{\pi}}$: $U^{\tilde{\pi}}\rightarrow \mathbb{R}^{N}$, such that $\phi^{\tilde{\pi}}(\mu_R,P) = V^{\tilde{\pi}}$ and for any  $ r'\times P'\in U^{\tilde{\pi}}$,
	\[ F^{\tilde{\pi}}(\phi^{\tilde{\pi}}(r', P'), r', P') =0. \]  
	In addition, the partial derivatives of $\phi^{\tilde{\pi}}$ satisfies
	\[ \left. \nabla \phi^{\tilde{\pi}}(\mu_R,P) = \frac{\partial \phi^{\tilde{\pi}}}{\partial (r',P')}\right |_{{r'=\mu_R}, {P'=P}} = \left. -\left[\frac{\partial F^{\tilde{\pi}}}{\partial V'} \right]^{-1} \left[ \frac{\partial F^{\tilde{\pi}}}{\partial r'}, \frac{\partial F^{\tilde{\pi}}}{\partial P'}\right] \right |_{{V'=V^{\tilde{\pi}}},{r'=\mu_R}, {P'=P}} \]
	where
	\[\left. \frac{\partial F^{\tilde{\pi}}}{\partial V'} \right|_{{V'=V^{\tilde{\pi}}},{r'=\mu_R}, {P'=P}}  = I-\gamma P^{\tilde{\pi}},\]\label{partialV}
	\[\left. \frac{\partial F^{\tilde{\pi}}}{\partial r'} \right|_{{V'=V^{\tilde{\pi}}},{r'=\mu_R}, {P'=P}} = - \begin{pmatrix} \tilde{\pi}(\cdot|1)^T  & & & & \\
	& \ddots & & &\\
	& & \tilde{\pi}(\cdot|i)^T & &\\
	& & & \ddots & \\ & & & & \tilde{\pi}(\cdot|m_s)^T \\
	\end{pmatrix},\]\label{partialr}
	where $\tilde{\pi}(\cdot|i)^T = [\tilde{\pi}(1|i), \dots, \tilde{\pi}(j|i), \dots, \tilde{\pi}(m_a|i)] $
	
	and
	\[\left. \frac{\partial F^{\tilde{\pi}}}{\partial P'} \right|_{{V'=V^{\tilde{\pi}}},{r'=\mu_R}, {P'=P}} = - \begin{pmatrix} (q^{\tilde{\pi}}_1)^T  & & & & \\
	& \ddots & & &\\
	& & (q^{\tilde{\pi}}_i)^T & &\\
	& & & \ddots & \\ & & & & (q^{\tilde{\pi}}_{m_s})^T \\
	\end{pmatrix},\]
	\label{partialP}
	where 
	$(q^{\tilde{\pi}}_i)^T =  \gamma [\tilde{\pi}(1|i)(V^{\tilde{\pi}})^T,\dots \tilde{\pi}(j|i)(V^{\tilde{\pi}})^T,\dots \tilde{\pi}(m_a|i)(V^{\tilde{\pi}})^T]$, 
	which is an $N$-dimensional vector.
	
	Applying the delta method, we have
	\begin{eqnarray*}
		\sqrt{n}(\hat{V}^{\tilde{\pi}}_n- V^{\tilde{\pi}})
		&=& \sqrt{n}(\phi^{\tilde{\pi}}(\hat{\mu}_{R,n}, \hat{P}_n) - \phi^{\tilde{\pi}}(\mu_R,P))\\
		&\Rightarrow& \mathcal{N}\left(0, \nabla \phi^{\tilde{\pi}}(\mu_R,P) \Sigma_{R,P} \left(\nabla \phi^{\tilde{\pi}}(\mu_R,P)\right)^T\right) \mbox{ as $n\to\infty$,}
	\end{eqnarray*}
	where 
	\[\nabla \phi^{\tilde{\pi}}(\mu_R,P) \Sigma_{R,P} \nabla \phi^{\tilde{\pi}}(\mu_R,P)^T=(I-\gamma P^{\tilde{\pi}})^{-1}  D^{\tilde{\pi}} \left((I-\gamma P^{\tilde{\pi}})^{-1} \right)^T.\]

\Halmos

\subsection{Proof of Theorem \ref{thm:non-unique}} \label{app:thm2}
To analyze MDPs with non-unique optimal policies, we utilize the LP representation of the MDP. Consider $\rho$ with $\rho(i)>0$, for $i=1,\dots, m_s$, the optimal value function $V^{*}$ is the optimal solution of of the following LP.	
	\begin{equation*}
	\begin{array}{ll}
	\min&\sum_{s}\rho(s) V(s)\\
	\text{subject to}& V(s) \geq \mu_R(s,a)+ \gamma \sum_{s'\in S} P(s'|s,a) V(s'),\ \forall s,a\\
	\end{array}\label{LP_prime}
	\end{equation*}	 
The dual of the LP takes the form
	\begin{equation*}
	\begin{array}{ll}
	\max&\sum_{s,a}\mu_R(s,a)x_{s,a}\\
	\text{subject to}
	&\sum_ax_{s,a}-\gamma\sum_{s',a}P(s|s',a)x_{s',a}=\rho(s),\ \forall s\\
	&x_{s,a}\geq0,\ \forall s,a
	\end{array}\label{LP_dual}
	\end{equation*}
where the decision variables, $x_{s,a}$'s, are known as the occupancy measure of the MDP \citep{puterman2014markov}.
If the MDP has more than one optimal policies, the dual problem has more than one optimal solutions, which further implies that the primal problem is degenerate. The degeneracy of the LP means that some constraints are redundant at the primal optimal solution, i.e., the optimal solution is at the intersection of more than $m_s$ hyperplanes.  Since the rows of the primal LP are linearly independent, in this case, there are multiple choices for the set of basic variables at the optimal solution. In addition, for any choice of $\rho$ with positive entries, $V^*$ is always the primal optimal solution \citep{puterman2014markov}. We denote $v_k$ as the $k$-th optimal basis and $A_k$ the corresponding column vectors, $1\leq k\leq \bar K$, for some $\bar K>1$.

For any $u\in U$, we denote the directional Jacobian of $V$ with respect to $P$ and $\mu_R$ as $D_u(P,\mu_R)$.
We next show that the directional Jacobian is well defined. 
Vectorize the matrix $(I-\gamma P)$ row by row and denote the corresponding $m_sN$-dimensional vector as $\Gamma_P$. 
We write $u=(u_P, u_R)\in \mathbb{R}^{m_sN+N}$ where $u_P\in \mathbb{R}^{m_sN}$ and $u_R\in \mathbb{R}^{N}$.
We also define the mapping from $\Gamma_P, \mu_R$ to the $k$-th optimal basis as $\psi_k$, $1\leq k\leq K$. In particular,
$v_k=\psi_k(\Gamma_P,\mu_R)=A_k^{-1}\mu_R$. We first note that $\psi_k(\Gamma_P+tu_P, \mu_R+tu_R)$ is continuous in $t$.
Next, fix $u\in U$, and consider the perturbed LP with constraint coefficients $\Gamma_P+tu_P$ and $\mu_R+tu_R$.
Due to the continuity of $\psi_k$'s,  there exists $t_u>0$, such that at least one basis is still optimal for all $t\in(0,t_u)$. 
In particular, there exists $k_u$, $1\leq k_u\leq \bar K$, such that $\psi_{k_u}(\Gamma_P+tu_P,\mu_R+tu_R)$ is the 
optimal solution to the perturbed LP for any $t\in(0,t_u)$. 
Lastly, we note that the directional derivative of $\psi_{k_u}$ at $(\Gamma_P,\mu_R)$ 
takes the form $u^T G_{k_u}$, where 
\[G_{k_u}=\left.\frac{\partial \psi_{k_u}}{\partial(\Gamma_P', \mu_R')}\right|_{\Gamma_P'=\Gamma_P, \mu_R'=\mu_R}.\]
This implies that $D_u(P,\mu_R) = u^T G_{k_u}$ is well defined.

Because there are $\bar K$ optimal basis, there are $K\leq \bar K$ distinct $G_k$'s.
Note that  we allow $K<\bar K$ since some $G_k$'s may be equal. 
With a little abuse of notation (which might involve rearranging the index), 
we write $G_1, \dots, G_K$ as the distinct $G_k$'s. 
Then, we can partition the set $U$ into $K$ subsets, $U_1, \dots, U_{K}$, 
such that if $u\in U_k$, $D_u(P,\mu_R)=u^T G_k $.

Lastly, define $\hat{u}_n = (\hat{P}_n-P, \hat{\mu}_{R,n}-\mu_R)/ \sqrt{||\hat{P}_n-P||^2 + ||\hat{\mu}_{R,n}-\mu_R||^2}$. 
We have
	\[\hat{V}^{*}_n - V^* = \sum^K_{k=1}G_k \mathds{1}\left( \hat{u}_n \in U_k \right) (\hat{P}_n-P, \hat{\mu}_{R,n}-\mu_R) + o_P(\|(\hat{P}_n-P, \hat{\mu}_{R,n}-\mu_R) \|).\] 
The convergence, 
$\sqrt{n}(\hat{V}_n^* - V^{*}) \Rightarrow \sum^K_{k=1} G_k  \mathds{1}\left( Z/\|Z\| \in U_k \right) Z$,
then follows from the continuous mapping theorem. \Halmos

\subsection{Proof of Theorem \ref{thm CMDP}}
	We use the LP representation of the constrained MDP. 
	Define $x_{s,a}$ as the occupancy measure
	\[x_{s,a}=\sum_{t=0}^\infty\gamma^t\PP(S_t=s|S_0\sim \rho),\]
	where $S_0\sim \rho$ denotes that the initial condition of the underlying Markov chain has probability mass function $\rho$.
	Then, $x_{s,a}$ satisfies the LP
	\begin{equation}
	\begin{array}{ll}
	\max&\sum_{s,a}\mu_R(s,a)x_{s,a}\\
	\text{subject to}&\sum_{s,a}\mu_C(s,a)x_{s,a}\leq\eta\\
	&\sum_ax_{s,a}-\gamma\sum_{s',a}P(s|s',a)x_{s',a}=\rho(s),\ \forall s\\
	&x_{s,a}\geq0,\ \forall s,a
	\end{array}\label{LP}
	\end{equation}
	(\eqref{LP} is the dual formulation in the proof of Theorem \ref{thm:non-unique} with an extra constraint.) 
	The objective and the first constraint correspond to the objective and the constraint in the constrained MDP formulation. The second constraint can be deduced by a one-step analysis on the definition of occupancy measure. 
	Let $(x_{s,a}^*)_{s,a}$ denote the optimal solution of the LP \eqref{LP}. Then, the optimal policy can be expressed as
	\[\pi^*(a|s)=\frac{x_{s,a}^*}{\sum_ax_{s,a}^*}.\]
	
	Note that the LP has $m_s+1$ constrains (excluding the non-negativity constrains). Thus, a basis has $m_s+1$ basic variables.
	Moreover, by our assumptions, the optimal solution is unique, which implies that perturbing the parameters $\mu_R,\mu_C$, and $P$ does not immediately imply an overshoot to negative values for the reduced costs of the non-basic variables. In particular, when the perturbation is small enough, we still retain the same optimal basis.
	Next, we consider two cases depending on whether the first constraint is binding or not. 
	
{\bf In the first case}, the first constraint is non-binding. In this case, the optimal policy is deterministic, i.e., for any $s$, $x_{s,a}>0$ for only one $a$. A small perturbation of the parameters still retains the same basic and non-basic variables, and the derived perturbed policy still retains the first constraint non-binding. The analysis then reduces to that of Corollary \ref{cor:V_basic_pi}.
	
{\bf In the second case}, the first constraint is binding.
	In this case, $x_{s,a}>0$ for only one $a$, for all $s$ except one state, $s_r$, where we can have $x_{s_r,a^*_1(s_r)}>0$ and $x_{s_r,a^*_2(s_r)}>0$ for two distinct actions $a^*_1(s_r),a^*_2(s_r)$. We denote the mixing parameter by $\alpha^{*} := \pi^{*}(a^*_1(s_r) |s_r)$, and so $\pi^{*}(a^*_2(s_r) |s_r) = 1- \alpha^{*}$. 
Again, perturbing the parameters retains the same basic and non-basic variables. In particular, the first constraint remains binding in the perturbation, and the perturbed optimal policy $\pi^*$ is still split at the same state and between the same actions.  

We next make a few observations for case 2. First, we define
\[\begin{split}
F_L^{\pi^*}(L,c,P',\alpha')=&L(s) - \left[c(s,a^*(s))+\gamma\sum_{s'}L(s')P(s'|s,a^*(s))\right]\mathds{1}(s\neq s_r) \\
&- \left[\left(c(s_r,a^*_1(s_r))+\gamma\sum_{s'}L(s') P(s'|s,a_1^*(s))\right)\alpha' \right.\\
&\quad \left.+ \left(c(s_r,a^*_2(s_r))+\gamma\sum_{s'}L(s') P(s'|s,a_2^*(s))\right)(1-\alpha')\right]\mathds{1}(s=s_r)
\end{split}\]
Then we have $F_L^{\pi^*}(L^{\pi^*}, \mu_C,P,\alpha^*)=0$ and $I-\gamma P^{\pi^*}$ is invertible.
Thus, by the implicit function theorem, there exists a continuously differentiable function $\phi_L$ such that  $L^{\pi^{*}}= \phi_L(\mu_C, P, \alpha^{*})$. 
Next, define
\[F_{\alpha}^{\pi^*}(c,P',\alpha')=\eta - \rho^T\phi_L(c, P', \alpha').\]
By applying the implicit function theorem again, we have there exists a continuously differentiable function $\phi_{\alpha}$ such that  $\alpha^*= \phi_{\alpha}(\mu_C, P)$. 
Following the same line of arguments as above, we define
\[\begin{split}
F_V^{\pi^*}(V',r',P',\alpha')=&V'(s) - \left[r'(s,a^*(s))+\gamma\sum_{s'}V'(s')P(s'|s,a^*(s))\right]\mathds{1}(s\neq s_r) \\
&- \left[\left(r'(s_r,a^*_1(s_r))+\gamma\sum_{s'}V'(s') P(s'|s,a_1^*(s))\right)\alpha' \right.\\
&\quad \left.+ \left(r'(s_r,a^*_2(s_r))+\gamma\sum_{s'}V'(s') P(s'|s,a_2^*(s))\right)(1-\alpha')\right]\mathds{1}(s=s_r)
\end{split}\]
Then, by the implicit function theorem,
$V^*$ can be viewed as a continuously differentiable function of $\mu_R$, $P$ and $\alpha^*$, i.e., $V^*=\phi_{I}(\mu_R, P, \alpha^*)$.
Since $\alpha^*$ can be viewed as a continuously differentiable function of $\mu_C$ and $P$, $V^*$ can also be viewed as a function of $\mu_R$, $\mu_C$, and $P$, i.e., $V^*=\phi_{V}(\mu_R, \mu_C, P)$.

Lastly, using the delta method, we have
\[\sqrt n(\hat V_n^*-V^*)\Rightarrow N(0,\nabla \phi_{V}(\mu_R, \mu_C, P)\Sigma_{R,C,P} \left(\nabla \phi_{V}(\mu_R, \mu_C, P)\right)^T) \mbox{ as $n\to \infty$},\]
where $\nabla \phi_{V}(\mu_R, \mu_C, P)$ is the Jacobian of $\phi_{V}$ evaluated at $\mu_R, \mu_C, P$,
and $\Sigma_{R,C,P}$ is the estimation covariance matrix of $\mu_R,\mu_C,P$.

To see where the explicit expression for $\Sigma_c$ is from, we note that
\[\frac{\partial \phi_{V}}{\partial (r',c', P')}
=\frac{\partial \phi_{I}}{\partial (r',c',P')} + \frac{\partial \phi_{I}}{\partial \alpha'} \frac{\partial \phi_{\alpha}}{\partial (r',c',P')}.\]
For $\phi_I$, we have
\[\left.\frac{\partial \phi_{I}}{\partial (r', P', \alpha')}\right|_{ r'=\mu_R, P'=P, \alpha'=\alpha^*}= \left. -\left[\frac{\partial F_V^{\pi^*}}{\partial V'} \right]^{-1} \left[ \frac{\partial F_V^{\pi^*}}{\partial r'}, \frac{\partial F_V^{\pi^*}}{\partial P'}, \frac{\partial F_V^{\pi^*}}{\partial \alpha}\right] \right |_{{V'=V^{*}},{r'=\mu_R}, {P'=P}, \alpha'=\alpha^*}\]
and \[\frac{\partial \phi_{I}}{\partial c'}={\bf 0},\]
where
\[\begin{split}
&\left.\frac{\partial F_V^{\pi^*}}{\partial V'}\right|_{{V'=V^{*}},{r'=\mu_R}, {P'=P}, \alpha'=\alpha^*}=I-\gamma P^{\pi^*}, ~~~ 
\left. \frac{\partial F_V^{\pi^*}}{\partial r'}\right|_{{V'=V^{*}},{r'=\mu_R}, {P'=P}, \alpha'=\alpha^*}=-G^{\pi^*},\\ 
& \left. \frac{\partial F_V^{\pi^*}}{\partial P'}\right|_{{V'=V^{*}},{r'=\mu_R}, {P'=P}, \alpha'=\alpha^*}=-H_V^{\pi^*}, 
\mbox{ and } \left. \frac{\partial F_V^{\pi^*}}{\partial \alpha'}\right|_{{V'=V^{*}},{r'=\mu_R}, {P'=P}, \alpha'=\alpha^*}=-h_V.
\end{split}\]
Similarly, for $\phi_{\alpha}$, we have
\[\left.\frac{\partial \phi_{\alpha}}{\partial (c', P')}\right|_{ c'=\mu_C, P'=P}= \left. -\left[\frac{\partial F_{\alpha}^{\pi^*}}{\partial \alpha'} \right]^{-1} \left[ \frac{\partial F_{\alpha}^{\pi^*}}{\partial c'}, \frac{\partial F_{\alpha}^{\pi^*}}{\partial P'}\right] \right |_{{\alpha'=\alpha^{*}},{c'=\mu_C}, {P'=P}}
\mbox{ and } \frac{\partial \phi_{\alpha}}{\partial r'}={\bf 0},\]
where
\[\begin{split}
&\left. \frac{\partial F_{\alpha}^{\pi^*}}{\partial \alpha'}\right |_{{\alpha'=\alpha^{*}},{c'=\mu_C}, {P'=P}} = -\rho^T(I-\gamma P^{\pi^*})^{-1} h_L, ~~~
\left. \frac{\partial F_{\alpha}^{\pi^*}}{\partial c'}\right|_{{V'=V^{*}},{r'=\mu_R}, {P'=P}, \alpha'=\alpha^*}=-\rho^T(I-\gamma P^{\pi^*})^{-1}G^{\pi^*},\\ 
&\mbox{ and }  \left. \frac{\partial F_{\alpha}^{\pi^*}}{\partial P'}\right|_{{V'=V^{*}},{r'=\mu_R}, {P'=P}, \alpha'=\alpha^*}=-\rho^T(I-\gamma P^{\pi^*})^{-1}H_L^{\pi^*}.
\end{split}\]
\Halmos

\section{Proof of the results in Section 4}
\subsection{Proof of lemma \ref{lemma:feasible set}}
	For any given policy $\pi$, by the balance equation for Markov Chains, its induced stationary distribution $w_\pi$ satisfies
	\[\sum_{k,l} w_{\pi}((k-1)m_a+l) P(i|s=k, a=l) \pi(a=j|s=i) = w_{\pi}((i-1)m_a+j)\]
	for any $i\in \mathcal{S}, j\in \mathcal{A}$.
	Summing up across $j$'s for each $i$, we have
	\begin{eqnarray*}
		&&\sum_{j} w_{\pi}((i-1)m_a+j) \\
		&=&  \sum_{j}\sum_{k,l} w_{\pi}((k-1)m_a+l) P(i|s=k, a=l) \pi(a=j|s=i)\\
		&=& \sum_{k,l} w_{\pi}((k-1)m_a+l)P(i|s=k, a=l)
	\end{eqnarray*} 
	On the other hand, for any $w$ in $\mathcal{W}$, $\pi_w$ satisfies
	\begin{eqnarray*}
		&&\sum_{k,l} w((k-1)m_a + l)  P(i|s=k, a=l) \pi_w(a=j|s=i)\\
		&=& \sum_{k,l} w((k-1)m_a + l)  P(i|s=k, a=l)  w((i-1)m_a+j)/ \sum_{u} w((i-1)m_a +u)\\
		&=& \sum_{u} w((i-1)m_a+u)  w((i-1)m_a+j)/ \sum_{u} w((i-1)m_a+u) = w((i-1)m_a+j)
	\end{eqnarray*} 
	for  all $ i\in S$. Thus, $w$ is the stationary distribution of the Markov chain with transition matrix $\tilde{P}^{\pi_w}$. \Halmos

\subsection{Proof of Theorem \ref{thm:Q-OCBA_convergence}}

We use the superscript $n$ to mark the dependence of $B_k$'s and $K$ on $n$, e.g., $B_k^n$ and $K^n$ when the sampling budget is $n$.

{\bf In the first case in the first result,} after collecting $B_{\hat k-1}^n$ data points, as $C_l>0$, $w_{\hat k}(s,a)>0$ for any $(s,a)\in\mathcal{S}\times\mathcal{A}$.
Then, as $B_{\hat{k}}^n - B_{\hat{k}-1}^n\to\infty$, each $(s,a)$ is visited infinitely often. By Theorem \ref{thm:clt_basic}, $\hat Q_{B_{\hat k}^n}\rightarrow Q$ a.s. as $n\rightarrow\infty$. Next, because $\hat c_{ij,B_{\hat k}^n}\rightarrow c_{ij}$ a.s. as $n\rightarrow\infty$, by continuous mapping theorem, we have $\hat\pi_{B_{\hat k}}\rightarrow\pi^e$ a.s. as $n\rightarrow\infty$.

{\bf In the second case in the first result,} because $C_l>0$, there exists $\epsilon>0$ such that $\min_{(s,a)\in \mathcal{S}\times\mathcal{A}}w_k(s,a)>\epsilon$ for all $k=1,2,\dots$. Because $B_k-B_{k-1}\geq N$, for any $(s,a)\in \mathcal{S}\times\mathcal{A}$, there is a positive probability of visiting $(s,a)$ at each stage. Thus, as $K\rightarrow\infty$, each $(s,a)$ is visited infinitely often. The rest follows the same line of argument as in the first case.

{\bf In the second result,} we first introduce a few more notations. Let $b_k=B_k-B_{k-1}$, i.e., the number of steps in the $k$-th iteration.
Let $s_{k,i}$, $1\leq i\leq b_k$, denote the state visited at the $i$-th step in the $k$-th iteration and $a_{k,i}$ denote the action taken then.

Note that $w_k$ is the stationary distribution of $\tilde P^{\pi_k}$, i.e., the Markov chain under the exploration policy for iteration $k$.
Then, by the large deviations theory for the empirical measure of a Markov chain (Theorem 1.2 in \cite{ellis1988large}), we have for any given $\epsilon>0$,
\[
\sup_{(s_{k,0},a_{k,0})\in\mathcal{S}\times\mathcal{A}} \mathbb{P}\left(\left.\sup_{(s,a)\in\mathcal{S}\times\mathcal{A}}\left|\frac{1}{b_k}\sum^{b_k}_{i=1} \mathds{1}(s_{k,i} = s, a_{k,i} =a) - w_k(s,a)\right|\geq \epsilon \right|s_{k,0},a_{k,0}\right) \leq \exp(-b_kI(\epsilon,w_k))
\]
for some large deviations rate function $I$, where $I(\epsilon,w_k)>0$ when $\epsilon>0$ and $I$ is continuous in $w_k$ ($I$ is continuous in $\tilde P^{\pi_k}$, which is in turn continuous in $w_k$). Because $w_k\rightarrow w^*$ as $k\rightarrow\infty$ and $I(\epsilon,w^*)>0$, there exists $\bar I(\epsilon)>0$ such that $\min_{k\geq 1}I(\epsilon,w_k)\geq \bar I(\epsilon)$. Then, as $b_k\geq \gamma k$,
\[\begin{split}
&\sum_{k=1}^{\infty}\sup_{(s_{k,0},a_{k,0})\in\mathcal{S}\times\mathcal{A}} \mathbb{P}\left(\left.\sup_{(s,a)\in\mathcal{S}\times\mathcal{A}}\left|\frac{1}{b_k}\sum^{b_k}_{i=1} \mathds{1}(s_{k,i} = s, a_{k,i} =a) - w_k(s,a)\right|\geq \epsilon \right|s_{k,0},a_{k,0}\right)\\
\leq&\sum_{k=1}^{\infty} \exp(-\gamma k\bar I(\epsilon))<\infty
\end{split}\]
By Borel-Cantelli Lemma, 
\[\frac{1}{b_k}\sum^{b_k}_{i=1} \mathds{1}(s_{k,i} = s, a_{k,i} =a) - w_k(s,a)  \rightarrow 0 \text{ a.s.\ as\ $k \to \infty$.}\]

Next, note that for any $(s,a)\in\mathcal{S}\times\mathcal{A}$,
\[\begin{split}
&\frac{1}{B_{K}}\sum^{K}_{k=1} \sum^{b_k}_{i=1} \mathds{1}(s_{k,i} = s, a_{k,i} =a) - w^*(s,a)\\
=& \sum^{K}_{k=1}\frac{b_k}{B_{K^n}} \left(\frac{1}{b_k}\sum^{b_k}_{i=1}  \mathds{1}(s_{k,i} = s, a_{k,i} =a) -w_k(s,a) +w_k(s,a)-w^*(s,a)\right) \\
=&\frac{1}{B_{K}}\sum^{K}_{k=1} \left(\frac{1}{b_k}\sum^{b_k}_{i=1}  \mathds{1}(s_{k,i} = s, a_{k,i} =a) -w_k(s,a)\right)b_k + \sum^{K}_{k=1}\frac{b_k}{B_{K}}(w_k(s,a)-w^*(s,a))
\end{split}\]
As 
\[\frac{1}{b_k} \sum^{b_k}_{i=1} \mathds{1}(s_{k,i} = s, a_{k,i} =a) - w_k(s,a)  \rightarrow 0 \mbox{ and }
w_k\rightarrow w^* \mbox{ a.s.\ as\ $k \to \infty$,}\]
by Silverman–Toeplitz Theorem,
\[
\frac{1}{B_{K}}\sum^{K}_{k=1} \sum^{b_k}_{i=1} \mathds{1}(s_{k,i} = s, a_{k,i} =a) - w^*(s,a) \rightarrow 0 \mbox{ a.s.\ as\ $n\rightarrow\infty$.}
\]
\Halmos

\section{Proof of the Results in Section 6}

\subsection{Proof of Theorem \ref{thm:clt_large}}
	Denote  
	\[[\hat{\mu}_{R,n}, \hat{P}_n]_{\mathcal{S}_0}=[\hat{\mu}_{R,n}((i-1)m_a+j), \hat{P}_n((i-1) N+(j-1)m_s + k)]_{i\in \mathcal{S}_0, 1\leq j\leq m_a, 1\leq k\leq m_s},\] and 
	\[[\mu_R, P]_{\mathcal{S}_0}=[\mu_R((i-1)m_a+j), P((i-1) N+(j-1)m_s + k)]_{i\in \mathcal{S}_0, 1\leq j\leq m_a, 1\leq k\leq m_s}.\] 
	By Assumption \ref{ass:large_state_stationary}, we have 
	\[\sqrt{n}([\hat{\mu}_{R,n}, \hat{P}_n]_{\mathcal{S}_0} - [\mu_R, P]_{\mathcal{S}_0}) \Rightarrow \mathcal{N}(0,\Sigma_{R, P, \mathcal{S}_0}) \mbox{ where }
	\Sigma_{R, P, \mathcal{S}_0} = 
	\begin{pmatrix} W_{\mathcal{S}_0}^{-1}D^{\mathcal{S}_0}_R  &  {\bf 0}  \\  {\bf 0} & D^{\mathcal{S}_0}_P \end{pmatrix}.\]
	
	Notice $M_g\circ M^{\mathcal{S}_0}_I\circ \mathcal{T}_{\hat{\mu}_{R,n}, \hat{P}_n}$ only involves random variables
	$[\hat{\mu}_{R,n}, \hat{P}_n]_{\mathcal{S}_0}$. Changing the distribution of $[\hat{\mu}_{R,n}, \hat{P}_n]_{\mathcal{S}\setminus \mathcal{S}_0}$ will not change the distribution of $\hat{Q}_n^M$. We can thus assign auxiliary random variables to  $\hat{\mu}_{R,n}$ and $\hat{P}_n$ for all $i \notin \mathcal{S}_0$, $1\leq j \leq m_a$, $1\leq k \leq m_s$. In particular, we use independent random variables for each $i\notin \mathcal{S}_0$ by letting
	\[\hat{\mu}_{R,n}((i-1)m_a+j) \overset{D}{=} \frac{1}{\sqrt{n}} \mathcal{N}(\mu_R((i-1) m_a+j), 1)\]
	\[\hat{P}_n((i-1)N+(j-1)m_s + k) \overset{D}{=} \frac{1}{\sqrt{n}} \mathcal{N}(P((i-1)N+(j-1)m_s + k), 1).\] 
	Doing so, we extend the $m_{s_0} m_a$-dimensional random variable $[\hat{\mu}_{R,n}, \hat{P}_n]_{\mathcal{S}_0}$ to an $m_s  m_a$-dimensional random variable $[\hat{\mu}_{R,n}, \hat{P}_n]_{\mathcal{S}}$ and 
	\[\sqrt{n}([\hat{\mu}_{R,n}, \hat{P}_n]_{\mathcal{S}} - [\mu_R, P]_{\mathcal{S}}) \Rightarrow \mathcal{N}(0,\Sigma_{R, P, \mathcal{S}}) \mbox{ where
		$\Sigma_{R, P, \mathcal{S}} = 
		\begin{pmatrix} \Sigma_{R, P, \mathcal{S}_0} & {\bf 0} \\ {\bf 0}  & I \end{pmatrix}$}.\]
	
	Similar to the proof of Theorem \ref{thm:clt_basic}, define 
	\[F_M(Q',r',P') = Q' - M_g \circ M_I^{\mathcal{S}_0}\circ \mathcal{T}_{r', P'}(Q').\] 
	By Assumption \ref{ass:large_state_unique}, $M_g$ is max-norm non-expansion. Then, $M_g \circ M^{\mathcal{S}_0}_I$ is also max-norm non-expansion, which implies that $\nabla (M_g \circ M^{\mathcal{S}_0}_I)$ has all its eigenvalues less than or equal to $1$. Thus,  
	\[\frac{\partial F_M}{\partial Q'} = \nabla M(\mathcal{T}_{r', P'}(Q'))(I - \gamma \tilde{P}')\] 
	is invertible. By Assumption \ref{ass:large_state_unique}, we have $F_M(Q^M,\mu_R,P) = 0$. By Assumption \ref{ass:large_state_unique2}, there exists a neighborhood $\Omega_M$ around $(Q^M,\mu_R,P)$, such that $F_M$ is continuously differentiable on $\Omega_M$.
	Then, applying the implicit function theorem, we have that there exists an open set $E_M \subset \Omega_M$ and a continuously differentiable function $\phi_M$ on $ E_M$, such that $\phi_M(\mu_R,P) = Q^M$  and 
	\[ \nabla \phi_M(\mu_R,P) = \left. -\left[\frac{\partial F_M}{\partial Q'} \right]^{-1} \left[ \frac{\partial F_M}{\partial r'}, \frac{\partial F_M}{\partial P'}\right] \right |_{{Q'=Q^M},{r'=\mu_R}, {P'=P}}.\]
	Using the delta method, we have
	\[\sqrt{n}(\hat{Q}^M_n - Q^M) \Rightarrow \mathcal{N}(0,\nabla \phi_M(\mu_R,P) \Sigma_{R, P, \mathcal{S}} (\nabla \phi_M(\mu_R,P))^T ) = \mathcal{N}(0, \Sigma^M_{\mathcal{S}_0} ) \mbox{ as $n\rightarrow\infty$.}\]
\Halmos

\subsection{Proof of Lemma \ref{thm:CLT_kernel_P}}
We have
\[\begin{split}
\sqrt{n} (\hat M_n - M^*) 
&= A_n^{-1} \sqrt{n} \left(\sum_{i=1}^n \phi(s_{i}, a_{i}) \psi(s^{\prime}_i) K_{\psi}^{-1} - \phi(s_{i}, a_{i})\phi^T(s_{i}, a_{i}) M^*\right) - \sqrt{n}A_n^{-1}M^{*}\\
&=\left(\frac{1}{n}A_n\right)^{-1}
\sqrt{n}\left(\frac{1}{n}\sum_{i=1}^{n}G(x_i)\right) -\left(\frac{1}{n}A_n\right)^{-1}\frac{1}{\sqrt{n}}M^*. 
\end{split}\]
We first note that as $n\rightarrow\infty$,
\[
\frac{1}{\sqrt{n}}M^{*} \Rightarrow 0
~\mbox{ and }~
\frac{1}{n}A_n \Rightarrow \mathbb{E}_{\bar w}[\phi(s(\bar X), a(\bar X))\phi^T(s(\bar X), a(\bar X))] = \sum_{s,a} w_{s,a} \phi(s, a)\phi(s, a)^T.
\]
We next note that because \[\mathbb{E}_{\bar w}[G_v(\bar X)]=
\sum_{(s,a)\in\mathcal{S}\times\mathcal{A}}w_{s,a}
\left(\phi(s,a)\phi(s,a)^TM^*\left(\sum_{s^{\prime}\in\mathcal{S}}\psi(s^{\prime})\psi(s^{\prime})^T\right)K_{\psi}^{-1} - \phi(s,a)\phi(s,a)^TM^*\right)
=0.
\]
By Markov Chain Central Limit Theorem, we have
\[
\sqrt{n}\left(\frac{1}{n}\sum_{t=1}^{n}G(\bar X_t)\right)
\Rightarrow \mathcal{N}(0, \Sigma_{G})
\mbox{ as $n\rightarrow\infty$,}
\]
where
\[\Sigma_{G} = \Var (G_v(X_0)) + 2\sum^{\infty}_{i=1} \Cov(G_v(X_0), G_v(X_i)).\]
%
%
%
Then, by Slutsky's theorem,
\[ \sqrt{n} (M_{v,n} - M^*_v) \Rightarrow N(0,\Upsilon_E^{-1}\Sigma_{G}\Upsilon_E^{-1}) \mbox{ as $n\rightarrow\infty$.}\] 
Because $\hat P_n^K(s^{\prime}|s,a)=F_{\phi,\psi}(s,a,s^{\prime})\hat M_{n,v}$,
\[\sqrt{n} (\hat P_n^K(s'|s,a) - P(s'|s,a)) \Rightarrow \mathcal{N}(0, F_{\phi, \psi}(s,a,s') \Upsilon_E^{-1}\Sigma_{G}\Upsilon_E^{-1}  (F_{\phi, \psi}(s,a,s'))^T) \mbox{ as $n\rightarrow\infty$.}\]
\Halmos

\subsection{Proof of Lemma \ref{thm:CLT_kernel_R}}
We have
\[\begin{split}
\sqrt{n}(\hat{\theta}_{\mu,n} - \theta_{\mu})
&=(\Phi^T_{\mu, n} \Phi_{\mu, n} )^{-1}\sqrt{n}\sum_{i=1}^{n}\phi_{\mu}(s_i,a_i)(r_i-\phi_{\mu}(s_i,a_i)^T\theta_{\mu})\\
&=\left(\frac{1}{n}\Phi^T_{\mu, n} \Phi_{\mu, n} \right)^{-1}\sqrt{n}\left(\frac{1}{n}\sum_{i=1}^{n}H(\tilde x_i)\right).
\end{split}\] 
We first note that
\[\frac{1}{n}\Phi^T_{\mu, n} \Phi_{\mu, n}\Rightarrow \Upsilon_{\mu} \mbox{ as $n\rightarrow\infty$}.\]
Next, because $\mathbb{E}_{w}[H(\tilde X)]=0$, by Markov Chain Central Limit Theorem, we have
\[
\sqrt{n}\left(\frac{1}{n}\sum_{i=1}^{n}H(\tilde x_i)\right)\Rightarrow \mathcal{N}(0, \Sigma_{H})
\mbox{ as $n\rightarrow\infty$,}
\]
where
\[
\Sigma_{H} = \Var_{w}(H(\tilde X_0)) + 2\sum^{\infty}_{t=1} \Cov_w(H(\tilde X_0), H(\tilde X_t)).
\]
Then, by Slutsky's theorem, 
\[ \sqrt{n}(\hat{\theta}_{\mu,n} - \theta_{\mu})\Rightarrow  \mathcal{N}(0, \Upsilon_{\mu}^{-1}\Sigma_{H}\Upsilon_{\mu}^{-1}) \]
Lastly, because $\hat\mu_{R,n}(s,a)=\phi_{\mu}(s,a)^T\hat\theta_{\mu,n}$,
\[ \sqrt{n}(\hat{\mu}_{R, n}(s,a) - \mu_R(s,a)) \Rightarrow N\left(0, \phi_{\mu}(s,a)^T\Upsilon_{\mu}^{-1}\Sigma_{H}\Upsilon_{\mu}^{-1}\phi_{\mu}(s,a)\right) \mbox{ as $n\rightarrow \infty$.}\]
\Halmos

\subsection{Proof of Theorem \ref{thm:CLT_kernel_Q}}
The proof of Theorem \ref{thm:CLT_kernel_Q} is similar to that of Theorem \ref{thm:clt_basic}.
We first note that under Assumption \ref{assum:ind}, 
by Lemmas \ref{thm:CLT_kernel_P} and \ref{thm:CLT_kernel_R},
\[
\sqrt{n}([\hat{\mu}_{R,n}, P_n] - [\mu_{R},P]) \Rightarrow \mathcal{N}(0, \Sigma^{K}_{R,P}),
\mbox{ as $n\rightarrow\infty$},
\]	
where 
\[
\Sigma^{K}_{R,P} = \begin{pmatrix} D^K_R &  {\bf 0} \\  {\bf 0} &  D^K_P \end{pmatrix}, 
~~~ D^K_P = F_{\phi, \psi} \Upsilon_E^{-1}\Sigma_{G}\Upsilon_E^{-1}  F^T_{\phi, \psi}, ~~~
D^K_R = \phi_{\mu}^T\Upsilon_{\mu}^{-1}\Sigma_{H}\Upsilon_{\mu}^{-1}\phi_{\mu}. 
\]
The covariance block ${\bf 0}$ in $\Sigma^{K}_{R,P}$ is due to the fact that under Assumption \ref{assum:ind},
\[
\mathbb{E}_{\bar w}\Big[\left(r(\bar X)-\phi_{\mu}(\bar X)^T\theta_{\mu}\right)
\left(\phi(s(\bar X), a(\bar X)) \psi(s^{\prime}(\bar X)) K_{\psi}^{-1} - \phi(s(\bar X), a(\bar X))\phi(s(\bar X), a(\bar X))^T M^*\right)\Big]
=0.
\]
Next, define $F(Q',r',P')$ as a mapping from $\mathbb{R}^{N}\times \mathbb{R}^{N}\times \mathbb{R}^{N m_s}$ to $\mathbb{R}^{N}$, representing the fixed point equation for the Bellman operator \eqref{eq:bellman}. Specifically,
\[\begin{split}
F(Q',r',P')((i-1)m_a+j) 
=& Q'((i-1)m_a+j) -r'((i-1)m_a+j)\\
&-\gamma \sum_{1\leq k\leq m_s}P'((i-1)N +(j-1)m_s+k) \text{ } g_k(Q'),
\end{split}\]  
for $1\leq i \leq m_s$ and $1\leq j \leq m_a$, where $g_k(Q') = \max_l Q'((k-1)m_a+l)$, for $1\leq k \leq m_s$.
We can then apply the implicit function theorem to the equation $F(Q,\mu_R, P) = 0$. 
In particular, there exists an open set $U\subset\mathbb{R}^{N} \times \mathbb{R}^{N m_s}$, around $(\mu_R,P)$, and a unique continuously differentiable function $\nu$: $U\rightarrow \mathbb{R}^{N}$, such that $\nu(\mu_R,P) = Q$, and for any $(r',P')\in U$
\[ F(\nu(r', P'), r', P') =0. \] 
Using the delta method, we have
\[	\sqrt{n}(\hat{Q}^K_n- Q)
= \sqrt{n}(\nu(\hat{\mu}_{R,n}, \hat P_n^K) - \nu(\mu_R,P))
\Rightarrow \mathcal{N}(0, \nabla \nu(\mu_R,P) \Sigma^K_{R,P} (\nabla  \nu(\mu_R,P))^T)\mbox{ as $n\to\infty$}.\]
Plugging in the formulas for $\nabla \nu(\mu_R,P)$ and $\Sigma^K_{R,P}$, we have
\[\nabla\nu(\mu_R,P) \Sigma^K_{R,P} \nabla\nu(\mu_R,P)^T=(I-\gamma \tilde{P}^{\pi^{*}})^{-1}
[D^K_R+\gamma^2V_DD_P^KV_D]((I-\gamma \tilde{P}^{\pi^{*}})^{-1})^{T}.
\]
\Halmos

\section{Optimal Exploration under a Different Objective}

In Q-OCBA, the sequential exploration policy is derived by maximizing the worst-case relative discrepancy among all Q-value estimates. 
If one is interested in obtaining the best estimate of $\chi^*$ (i.e., the optimal value function initialized according to $\rho$), then it would be more natural to solve
\begin{equation}
\min_{w\in\mathcal{W}_{\eta}} \sigma_{\chi}^2\label{var min}.
\end{equation}
From \eqref{var min}, the optimal exploration policy $\pi_w$ can be derived based on Lemma \ref{lemma:feasible set}. The motivation is that by doing so, the half-width of the CI for $\chi^*$ is minimized. We refer to this policy as $\chi^*$-OCBA

Table \ref{tab:CI_len_comparison} compares the half-width of $95\%$ CI for $\chi^*$ constructed using data collected under different exploration policies. 
The half-width is estimated based on $10^3$ independent replications of the procedures and the estimation errors (half width of the 95\% CIs) are around $0.01$.
We vary $r_L$ from $1$ to $3$. 
The total sampling budget is $n=10^4$. 
$30\%$ of the budget is used to collect some initial data under RE with $\pi(1|s)=0.8$. We then use the initial data to train $\epsilon$-greedy with $\epsilon=0.2$, UCRL, PSRL, Q-OCBA, and $\chi^*$-OCBA. This two-stage implementation creates a warm start to help UCRL and PSRL achieve the correct coverage rate, i.e., $95\%$. In particular, using this implementation, all policies in Table \ref{tab:CI_len_comparison} are able to achieve the nominal coverage.

We observe that $\chi^*$-OCBA leads to shorter CIs in $\chi^*$ estimation compared to benchmark policies. 
For example, when $r_L=3$, the half-width of the CI constructed from $\chi^*$-OCBA is at least $40\%$ less than those constructed using PSRL or UCRL. The performances of $\chi^*$-OCBA and Q-OCBA appear quite similar, with $\chi^*$-OCBA moderately better. For example, when $r_L=3$, the half-width of the CI constructed from $\chi^*$-OCBA is $24\%$ less than that using Q-OCBA. We attribute the improvement to the targeted criterion on $\chi^*$ estimation used in our policy design of $\chi^*$-OCBA.
We also note that the performance of $\chi^*$-OCBA is robust against different values of $r_L$.
This is not the case for some benchmark policies
such as PSRL, which is designed to strike a balance between exploration and exploitation. 
For example, when $r_L=1$, PSRL can perform as well as Q-OCBA, with the half-width of the CI being $2.16$ for PSRL, and $2.2$ for $\chi^*$-OCBA. 
However, when $r_L=2$ or $3$, $\chi^*$-OCBA demonstrates significant advantages over PSRL. Specifically, the half width of the CI using $\chi^*$-OCBA is less than $58\%$ of that using PSRL.  
This observation is consistent with our observation from Figure \ref{fig:seqm} where we consider the probability of correct selection.

\begin{table}[H] 
	\caption{Comparison of CI half-widths for different exploration policies} \label{tab:CI_len_comparison}
	\centering
		\begin{tabular}{c|c|c|c|c|c|c}\hline
			$r_L$ &  $0.2$-greedy & RE(0.6) & PSRL & UCRL &  Q-OCBA & $\chi^*$-OCBA  \\ \hline
			
			1 & $3.50$ & $3.47$ & $2.16$ & $6.25$&  $2.58$ & $2.20$   \\ 
			2 & $3.12$ &$2.07$  & $3.26$ &$3.31$ &  $1.38$ & $1.35$   \\ 
			
			3 &$1.47$  & $1.48$
			&  $1.40$ &$1.40$ & $1.06$ & $0.81$    \\ \hline
		\end{tabular}
\end{table}

\end{APPENDICES}

\bibliographystyle{informs2014}
\bibliography{Q_OCBA_ref1.bib}

\end{document}